\pdfoutput=1

\documentclass[11pt]{article}

\usepackage{ACL2023}

\usepackage{times}
\usepackage{latexsym}

\usepackage[T1]{fontenc}

\usepackage[utf8]{inputenc}

\usepackage{microtype}
\usepackage{amsmath}
\usepackage{amssymb}
\usepackage{mathtools}
\usepackage{amsthm}
\usepackage{enumitem}
\usepackage{smile}
\usepackage[capitalize,noabbrev]{cleveref}
\usepackage{inconsolata}

\theoremstyle{plain}
\newtheorem{theorem}{Theorem}[section]

\newtheorem{lemma}[theorem]{Lemma}
\newtheorem{corollary}[theorem]{Corollary}
\theoremstyle{definition}
\newtheorem{definition}[theorem]{Definition}
\newtheorem{assumption}[theorem]{Assumption}
\theoremstyle{remark}
\newtheorem{remark}[theorem]{Remark}

\title{Federated Learning with Projected Trajectory Regularization}

\author{Tiejin Chen$^*$ \\
  Arizona State University \\
  \texttt{tiejin@asu.edu} \\
  \And
  Yuanpu Cao$^*$ \\
  The Pennsylvania State University \\
  \texttt{ymc5533@psu.edu} \\
  \AND
  Yujia Wang$^*$ \\
  The Pennsylvania State University \\
  \texttt{yjw5427@psu.edu} \\
  \And
  Cho-Jui Hsieh \\
  UCLA \\
  \texttt{chohsieh@cs.ucla.edu} \\
  \And
  Jinghui Chen \\
  The Pennsylvania State University \\
  \texttt{jzc5917@psu.edu} \\  
  }

\begin{document}

\maketitle
\def\thefootnote{*}\footnotetext{Equal contribution.}
\def\thefootnote{\arabic{footnote}}

\begin{abstract}
Federated learning enables joint training of machine learning models from distributed clients without sharing their local data. One key challenge in federated learning is to handle non-identically distributed data across the clients, which leads to deteriorated model training performances. Prior works in this line of research mainly focus on utilizing last-step global model parameters/gradients or the linear combinations of the past model parameters/gradients, which do not fully exploit the potential of global information from the model training trajectory. In this paper, we propose a novel federated learning framework with projected trajectory regularization (FedPTR) for tackling the data heterogeneity issue, which proposes a unique way to better extract the essential global information from the model training trajectory. Specifically, FedPTR allows local clients or the server to optimize an auxiliary (synthetic) dataset that mimics the learning dynamics of the recent model update and utilizes it to project the next-step model trajectory for local training regularization. We conduct rigorous theoretical analysis for our proposed framework under nonconvex stochastic settings to verify its fast convergence under heterogeneous data distributions. Experiments on various benchmark datasets and non-i.i.d. settings validate the effectiveness of our proposed framework.
\end{abstract}

\section{Introduction}

Federated Learning (FL) \cite{mcmahan2017communication, konevcny2016federated} has recently emerged as an attractive learning paradigm where a large number of remote clients (e.g., mobile devices, institutions, organizations) cooperate to jointly learn a machine learning model without data sharing. Specifically, Federated Learning works by performing local training on each client based on their own data, while having the clients communicate their local model updates to a global server \cite{mcmahan2017communication, konevcny2016federated}. 
Specifically, FedAvg \cite{mcmahan2017communication} updates the global model by averaging multiple steps of local stochastic gradient descent (SGD) updates. Though FedAvg has demonstrated empirical success on independent and identically distributed (i.i.d) client data, when facing the more practical non-i.i.d data distributions across the clients, it can lead to diverging updates and severe drift between the local and the global training objective in each client, which largely deteriorates the model training performance \cite{khaled2020tighter, karimireddy2020scaffold,wang2020tackling}.

To tackle this data heterogeneity issue, a line of research mainly focuses on utilizing or regularizing with the last-step global model parameters/gradients or the linear combinations of past model parameters/gradients \cite{li2020federatedprox,karimireddy2020scaffold,gao2022feddc}. 
Specifically, FedProx \cite{li2020federatedprox} makes a simple modification to FedAvg by adding a proximal regularization term in the local training objective to restrict the local updates to be closer to the last-step global model. SCAFFOLD \cite{karimireddy2020scaffold} reduces client drift with a control variate that linearly accumulates past gradients and corrects for the drift. FedDC \cite{gao2022feddc} follows a similar idea as SCAFFOLD with a local drift term and introduces a linear gradient correction loss, attempting to further eliminate the drift at local training. 
However, such methods may not be sufficient enough to fully exploit global information to tackle extreme data heterogeneity. In order to better extract the essential global information from the model training trajectory, we propose a novel federated learning framework with projected trajectory regularization (FedPTR), 
to help guide the local model update. Specifically, during the local training, FedPTR allows
local clients or the server to optimize an auxiliary dataset that mimics the learning dynamics of the recent model update and then utilizes it to project the next-step model trajectory for local training regularization. 
To achieve the above goal and optimize the auxiliary dataset, we adopt the Matching Training Trajectories (MTT) technique \cite{cazenavette2022dataset}, which was originally proposed for the task of dataset distillation. Here we want to emphasize that our usage of the Matching Training Trajectories is actually different from traditional dataset distillation whose goal is to synthesize a small dataset such that models trained on the synthetic dataset can achieve similar test accuracy as if trained on the actual
dataset. Firstly, most works in dataset distillation cannot fit in here as they usually require access to the entire training dataset which is not practical for federated learning.
Some existing works \cite{zhou2020distilled,song2022federated} aim to get around this by letting local clients distill their own data and send it to the server. However, this could incur extra communication costs and raise data privacy concerns. Secondly, the goal of dataset distillation is quite ambitious and usually requires obtaining multiple model training trajectories (or re-starts) in order to obtain barely satisfactory performance. In sharp contrast, here we only require the extracted auxiliary dataset to mimic the learning dynamics of recent model updates for projecting the next-step update, which is a much easier task that can be achieved with only one global model train trajectory.

We summarize our contribution as follows:
 \begin{itemize}[leftmargin=*]
     \item We propose a novel federated learning framework FedPTR that handles extreme data heterogeneity across local clients. The proposed method allows each client or the server to extract essential global information from the observed global model training trajectory and utilize it to project the next-step model trajectory for local training regularization.
     \item We provide rigorous theoretical analysis and prove that FedPTR can achieve a convergence rate of $\mathcal{O}(1/T)$ w.r.t total communication rounds $T$ under stochastic non-convex optimization settings, which matches the best existing results.
     \item We conduct extensive experiments on various benchmark datasets and non-i.i.d settings to demonstrate the effectiveness of our proposed FedPTR in training real-world machine learning models.
 \end{itemize}

\section{Related Work}
\noindent\textbf{Federated Learning.}
Federated Learning \cite{mcmahan2017communication} has been widely applied to all kinds of clients (e.g., mobile devices, institutions, and organizations) to collaboratively train a global model without sharing their private data. However, the heterogeneous data distribution across different devices comes to be one of the major bottlenecks for Federated Learning algorithms. Based on a similar idea as FedAvg \cite{mcmahan2017communication}, several federated optimization methods considered server-level momentum update such as FedAvgM \cite{hsu2019measuring}, FedAGM \cite{kim2022communication}, FedCM \cite{xu2021fedcm} or global-wise adaptive update such as FedAdagrad \cite{reddi2020adaptive}, FedYogi \cite{reddi2020adaptive}, FedAdam \cite{reddi2020adaptive}, FedAMS \cite{wang2022communication} which take advantage of both stable optimization methods and the federated learning paradigm. A similar strategy to improve the stability of federated learning optimization is adopting client-level optimization methods, including FedSAM and MoFedSAM \cite{qu2022generalized}, MOON \cite{li2021model} and FedPD \cite{zhang2020fedpd}. Moreover, personalized federated learning has also been discussed as the potential distribution across clients could be non-i.i.d. Several personalized federated learning approaches include meta-learning-based methods such as Per-FedAvg \citep{fallah2020personalized}, multi-task-learning-based method MOCHA \citep{smith2017federated} and hybrid global and local models such as APFL \citep{deng2020adaptive}, Ditto \citep{li2021ditto}, pFedMe \citep{t2020personalized}, L2SGD \citep{hanzely2020federated}, etc. 

\noindent\textbf{Federated Learning for Heterogeneous Data.}
Recent work showed that adding a proximal or regularization term during training helps mitigate the impact of statistical heterogeneity. FedProx \cite{li2020federatedprox} added a proximate term to help alleviate both system and statistical heterogeneity and improve model stability. FedDyn \cite{acar2021federated} proposed a framework with local dynamic regularization terms to align with the global model. FedAGM \cite{kim2022communication} introduced server gradient acceleration methods with regularization. Some others introduced extra control variables for alleviating data heterogeneity among devices. SCAFFOLD \cite{karimireddy2020scaffold} and Mime \cite{karimireddy2020mime} updated and communicated extra historical control variables during local updates to reduce the variance across clients. FedDC \cite{gao2022feddc} involved both dynamic regularization terms and local drift variables for model correction.

\noindent\textbf{Dataset Distillation.}
Dataset distillation is proposed to learn a small synthetic dataset out of the original large dataset, such that the model trained on this small dataset can have a similar performance as the model trained on the original dataset. \cite{wang2018dataset} first introduced dataset distillation by gradient-based hyperparameter optimization. \cite{bohdal2020flexible} proposed a label distillation method to get soft labels for distilled images with the meta-learning framework.\cite{zhao2021dataset,zhao2021dataset_sa} proposed to match the gradient of the real dataset and the auxiliary dataset, and use data augmentation to boost the performance. \cite{wang2022cafe,zhao2021dataset_dm} considered matching the distribution based on neural networks and largely reduced the memory cost. Recently, Matching Training Trajectories \cite{cazenavette2022dataset} and its variants \cite{cazenavette2022dataset,liu2022dataset,du2022minimizing} are proposed to optimize the synthetic data by minimizing the distance between the synthetically trained parameters and the actual model training trajectories.

\noindent\textbf{Federated Learning with Dataset Distillation.}
Several recent works have tried to combine federated learning with dataset distillation. Specifically, \cite{zhou2020distilled,song2022federated} proposed a one-shot federated learning framework that allows all clients to distill their own dataset by local training and send the synthetic dataset to the server for centralized training. \cite{xiong2022feddm} proposes a multi-shot framework that clients will distill their dataset for every communication round. Note that the above-mentioned works require (distilled) data transmission between clients and server: clients need to send their local distilled data to the server. Yet this incurs extra communication costs and data privacy concerns. In sharp contrast, our algorithm maintains the same basic framework as traditional federated learning methods, where only the model parameters are synced between the server and clients.

\section{Proposed Method} \label{Methods}
In this section, we introduce our proposed
federated learning framework with Projected Trajectory Regularization (FedPTR) for tackling the data heterogeneity issue in federated learning.  

\noindent\textbf{Preliminaries on Federated Learning.} 
Consider a federated system with $N$ total clients participated, the goal of federated learning is to optimize the following objective function: 

\begin{equation}\label{eq:fed-obj}
\min_{\wb \in \RR^d} \sum_{i=1}^{N} \frac{m_i}{M} \mathcal{L}(\mathbf{w}; \mathcal{D}_i),
\end{equation}
where $\mathcal{L}$ denotes the loss function such as Cross Entropy loss, $\mathcal{D}_{i}$ denotes the training data on client $i$.   Let $m_i = |\mathcal{D}_{i}|$ be the cardinality of $\mathcal{D}_i$ and we have $M = \sum_{i=1}^N m_i $ as total size of the training dataset. 
FedAvg \cite{mcmahan2017communication} is designed to solve Eq \eqref{eq:fed-obj} by coordinating multiple devices via a central server. Specifically, in each round of FedAvg, the server will broadcast the current global model $\wb^t$ to the clients.
Each client will locally optimize the corresponding component function in Eq \eqref{eq:fed-obj} and obtain the local model $\mathbf{w}_i^{t+1}$, i.e., $\mathbf{w}_i^{t+1} = \argmin_{\mathbf{w}} \mathcal{L}(\mathbf{w};\mathcal{D}_i)$, through multiple steps of stochastic gradient descent.
The server will then aggregate (weighted average) the local models $\{\mathbf{w}_i^{t+1}\}_{i=1}^N$ to obtain the new global model. 

FedProx \cite{li2020federatedprox} is another representative work which is designed to tackle both the data heterogeneity issue in federated learning and allows flexible performances on each of the local clients.
Specifically, it allows the local objectives to be solved inexactly such that the amount of local computation/communication can be adapted for different clients with different amounts of resources. To formally define this inexactness, \cite{li2020federatedprox} introduced the following definition:
\begin{definition}[$\gamma$-Inexactness]\label{as:gamma-inexact}
Given a function $h(\wb, \wb_0) = \cL (\wb) + \frac{\lambda}{2} \|\wb-\wb_0\|_2^2$ and $\gamma \in [0,1]$, $\wb^*$ is a $\gamma$-inexact solution of $\min_{\wb} h(\wb, \wb_0)$ if $\|\nabla h(\wb^*, \wb_0)\|_2 \leq \gamma \|\nabla h(\wb_0, \wb_0)\|_2$.
\end{definition}
Additionally, FedProx introduced a proximal term to the local objective to limit the local updates to be not too far away from its initial global model:
$
\mathbf{w}_i^{t+1} = \argmin_{\wb}  \mathcal{L}(\mathbf{w};\mathcal{D}_i) + \frac{\lambda}{2} \lVert \mathbf{w} - \mathbf{w}^t \rVert^2_2,
$
where $\lambda$ controls the strength of the proximal term.  We can observe that FedProx handles data heterogeneity by restricting the local updates to stay close to the initial global model.

\subsection{Proposed Method}

Due to the privacy-preserving design of federated learning, each client can only access their own training data. Therefore, the global model training trajectories contain all the information each client could possibly access with regard to other clients' local training data. 
Naturally, most existing works seek help from global models or global variables to mitigate the data heterogeneity issue.
Different from existing works which mainly focus on utilizing only the last-step global model or the linear combination of past global models, FedPTR aims to provide a better way to exploit the global information hidden in the model training trajectories via Matching Training Trajectories (MTT) \cite{cazenavette2022dataset}. More specifically, we aim to extract an auxiliary dataset from the recent model training trajectories that could mimic the learning dynamics of the recent model update.
Thus we can utilize such an auxiliary dataset for projecting the next-step training trajectory and provide better guidance to the local model training.

\noindent\textbf{Extract Global Information via MTT.} We found Matching Training Trajectories (MTT) \cite{cazenavette2022dataset} to be a perfect tool for our global information extraction purpose. 
The general goal of MTT is to learn an informative but small synthetic dataset that approximates the true learning trajectories. 
\begin{algorithm}[H]
\small
   \caption{Matching Training Trajectories (MTT)}
   \label{MTT}
\begin{algorithmic}[1]
   \STATE {\bfseries Input:}  ${\mathbf{w}}^{\text{start}}$,  ${\mathbf{w}}^{\text{end}}$, $\tilde{\mathcal{D}}_i$. 
   \FOR{$h=0$ {\bfseries to} $H-1$}
   \STATE $\hat{\mathbf{w}} = {\mathbf{w}}^{\text{start}}$  
   \FOR{$r=0$ {\bfseries to} $R-1$}
   \STATE $\hat{\mathbf{w}} \leftarrow \hat{\mathbf{w}} - \beta \nabla \mathcal{L}(\hat{\mathbf{w}};\tilde{\mathcal{D}}_i)$
   \ENDFOR
   \STATE $\mathcal{L}_{\text{MTT}}(\tilde{\mathcal{D}}_i) = \lVert \hat{\mathbf{w}} - {\mathbf{w}}^{\text{end}}\rVert^2_2 / \lVert {\mathbf{w}}^{\text{end}} - {\mathbf{w}}^{\text{start}}\rVert^2_2$
   \STATE Updates $\tilde{\mathcal{D}}_i$ with respect to $\mathcal{L}_{\text{MTT}}(\tilde{\mathcal{D}}_i)$
   \ENDFOR
   \STATE {\bfseries Output:} { $\mathcal{\tilde{D}}_i$}
\end{algorithmic}
\end{algorithm}
Moreover, it only requires model parameters from previous training epochs which naturally fit the setting of federated learning. We summarize the detailed MTT process in \cref{MTT}. 
Specifically, we treat the last model in the actual training trajectories (denoted by ${\mathbf{w}}^{\text{end}}$) as the teacher model in MTT and build a student model $\hat{\mathbf{w}}$, which is obtained by multiple steps of gradient descent training on the synthetic dataset $\tilde\cD$ starting from ${\mathbf{w}}^{\text{start}}$ in the trajectory. 
By requiring the student model to stay close to the teacher model, we force MTT to learn a synthetic dataset $\tilde\cD$ that well-represents the recent learning dynamics. In detail, we have the following MTT loss: 
\begin{align*}
\mathcal{L}_{\text{MTT}}(\tilde{\mathcal{D}}) = \lVert \hat{\mathbf{w}} - {\mathbf{w}}^{\text{end}}\rVert^2_2 / \lVert {\mathbf{w}}^{\text{end}} - {\mathbf{w}}^{\text{start}}\rVert^2_2,
\end{align*}
where 
we can optimize $\mathcal{L}_{\text{MTT}}(\tilde{\mathcal{D}})$ over $\tilde{\mathcal{D}} $ via gradient descent to obtain the auxiliary synthetic dataset.

\noindent\textbf{Projected Trajectory Regularization.} 
Once we obtain the auxiliary dataset via MTT, the following step is to project the next-step training trajectory and design our federated learning framework. We summarize our design at \cref{FedPTR} and give more details in the following \footnote{\cref{FedPTR} focused on the full participation case where all clients participate in each round. For partial participation cases, we can simply allow selected clients to perform MTT from the past trajectories it observed.}.  
Specifically, for each client (after certain training rounds $m$), we utilize MTT for extracting the auxiliary dataset $\tilde\cD_i$. Since the auxiliary dataset is designed to mimic the training trajectory from the $(t-m)$-th round global model to the $t$-th round global model, we can obtain our projected next-step model training trajectory  $\tilde{\mathbf{w}}_i^{t+1}$ by further performing several steps of gradient descent training on $\tilde\cD_i$ with learning rate $\eta$. 

Inspired by FedProx, we also design a proximal term for guiding the local training. Different from FedProx, we don't use the last-step global model $\xb^t$, instead, we hope the local update to stay close to our projected training trajectory  $\tilde{\mathbf{w}}_i^{t+1}$: 
\begin{equation}\label{eq:fedptr}
   \min_{\mathbf{w}} \mathcal{L}(\mathbf{w};\mathcal{D}_i) + \frac{\lambda}{2} \lVert \mathbf{w} - \tilde{\mathbf{w}}_i^{t+1} \rVert^2_2,
\end{equation}
where $\lambda$ controls the strength of the proximal term. Through such a design, we utilized the extracted global information from past training trajectories to help guide the local model update.   
To make our algorithm more practical and provide more flexibility for local clients, we also follow FedProx and only require a $\gamma$-inexactness solution during local training.  
After finishing the local computation, clients will send their updated model to the server and the rest steps follow the same procedure as vanilla FedAvg.  

\noindent\textbf{Initialization of the Auxiliary Dataset.}
Previous studies \cite{cazenavette2022dataset}  have shown that using real data samples from true classes can improve the performance of MTT. However, it is not very practical in federated learning with heterogeneous data distribution, as it is nearly impossible for individual clients to possess real data samples from each class. 
Hence, we use an alternative strategy to initialize the auxiliary dataset $\tilde{\mathcal{D}}_i$. Specifically, for those classes that have corresponding samples in the local training data, we directly random sample (with replacement) the desired amount of real data from the local training data. For those classes that do not have samples in the local training data, we adopt the random initialization strategy.  

\begin{algorithm}[htb]
   \caption{Federated Learning with Projected Trajectory Regularization (FedPTR)}
   \label{FedPTR}
   \small
\begin{algorithmic}[1]
   \STATE {\bfseries Input:} $T$, $m$, $\eta$, $\lambda$.
    \STATE Initialize $\mathcal{\tilde{D}}_i$ for all clients
   \FOR{$t=0$ {\bfseries to} $T-1$}
   \FOR{each client $i \in [N]$}
   \STATE Client $i$ receives $\mathbf{w}^{t}$ from the server
   \STATE $\tilde{\mathbf{w}}^{t+1}_i = \wb^{t}$
   \IF{$t > m$}
   \STATE \colorbox{blue!10}{$\mathcal{\tilde{D}}_i \leftarrow $ MTT$( \mathbf{w}^{t-m}, \mathbf{w}^{t}, 
 \mathcal{\tilde{D}}_i)$}
   \FOR{$k=0$ {\bfseries to} ${K}-1$}
   \STATE \colorbox{blue!10}{$\tilde{\mathbf{w}}^{t+1}_i \leftarrow \tilde{\mathbf{w}}^{t+1}_i - \eta \nabla \mathcal{L}(\tilde{\mathbf{w}}^{t+1}_i;\tilde{\mathcal{D}}_i)$}
   \ENDFOR
   \STATE Set $\tilde{\lambda} = \lambda$
   \ELSE
   \STATE Set $\tilde{\lambda} = 0$
   \ENDIF
   
   \STATE Solve $\mathbf{w}_i^{t+1}$ as a  $\gamma$-inexactness minimizer: \\{\scriptsize$\mathbf{w}_i^{t+1} \approx \mathop{\mathrm{argmin}}\limits_{\mathbf{w}} \mathcal{L}(\mathbf{w};\mathcal{D}_i) + \frac{\tilde{\lambda}}{2}\cdot \lVert \mathbf{w} - \tilde{\mathbf{w}}_i^{{t+1}} \rVert_2^2$}
   \STATE Client sends $\mathbf{w}_i^{t+1}$ to server
   \ENDFOR
   \STATE Server aggregates $\{\mathbf{w}_i^{t+1}\}_{i=1}^N$ to get $\mathbf{w}^{t+1}$
    \ENDFOR
\end{algorithmic}
\end{algorithm}

\noindent\textbf{Adaptive Regularization Parameter.}
In Eq \eqref{eq:fedptr}, we use a standard squared $L_2$ norm regularization as the proximal term.  Following common machine learning practices, we usually adopt a constant $\lambda$ as the regularization parameter. However, one thing we notice is that in practice, when applied this $L_2$ norm regularization in neural network models, the model difference $\mathbf{w} - \tilde{\mathbf{w}}_i^{t+1}$ can vary drastically among different convolutional layers\footnote{see more details in Appendix \ref{app:layer}}. 
This motivates us to use a layer-adaptive $\lambda$ rather than a constant. Specifically, let $\wb_{[j]}$ denotes parameters for the $j$-th layer in $\mathbf{w}$ and $\tilde\wb^{t+1}_{i,[j]}$ denotes parameters for the $j$-th layer in $\mathbf{\tilde{w}}$, we can define the adaptive $\lambda_j$ as follows:
$ 
\lambda_j = \lambda/\lVert \wb_{[j]} -\tilde\wb^{t+1}_{i,[j]} \rVert_2.
$In practice, we find that using layer adaptive $\lambda$ can increase the performance and we report all our experimental results in Experiments with the layer adaptive $\lambda$ by default.

\noindent\textbf{Server-Side FedPTR.} One problem of FedPTR presented in \cref{FedPTR} is that each client needs to perform Matching Training Trajectories and update their own auxiliary dataset in each communication round, which would introduce extra computation costs. In order to make our method smoothly accommodate resource-constrained devices (e.g., mobile phones, robots, drone swarms), we also provide a server-side FedPTR algorithm (denoted by FedPTR-S), which only optimizes the auxiliary dataset through MTT in the global server as shown in \cref{FedPTR_server_distill}. 
Specifically, in each communication round, the global server first starts to match training trajectories and optimize a unified $\tilde{\mathbf{w}}^{t+1}$. Then, each client will receive both $\mathbf{w}^{t}$ and $\tilde{\mathbf{w}}^{t+1}$ for local training. Note that different from client-side FedPTR as shown in \cref{FedPTR}, the server-side version initializes the auxiliary dataset using random noise without any real data from local clients. Therefore, our FedPTR-S can preserve data privacy while also making participating devices free from extra computation burdens.
\begin{algorithm}[htb]
   \caption{Server-Side FedPTR (FedPTR-S)}
   \label{FedPTR_server_distill}
   \small
\begin{algorithmic}[1]
   \STATE {\bfseries Input:} $T$, $m$, $\eta$, $\lambda$.
   \STATE Initialize $\mathcal{\tilde{D}}$ for server
   \FOR{$t=0$ {\bfseries to} $T-1$}
   \STATE Server starts to match training trajectories:
   \IF{$t > m$}
   \STATE \colorbox{blue!10}{$\mathcal{\tilde{D}} \leftarrow $ MTT$( \mathbf{w}^{t-m}, \mathbf{w}^{t}, 
 \mathcal{\tilde{D}})$}
   \FOR{$k=0$ {\bfseries to} ${K}-1$}
   \STATE \colorbox{blue!10}{$\tilde{\mathbf{w}}^{t+1} \leftarrow \tilde{\mathbf{w}}^{t+1} - \eta \nabla \mathcal{L}(\tilde{\mathbf{w}}^{t+1};\tilde{\mathcal{D}})$} 
   \ENDFOR
   \STATE Set $\tilde{\lambda} = \lambda$
   \ELSE
   \STATE $\tilde{\mathbf{w}}^{t+1} = \wb^{t}$
   \STATE Set $\tilde{\lambda} = 0$
   \ENDIF

   \FOR{each client $i \in [N]$}
   \STATE Client $i$ receives $\mathbf{w}^{t}$ and $\tilde{\mathbf{w}}^{t+1}$ from server
   
   \STATE Solve $\mathbf{w}_i^{t+1}$ as a  $\gamma$-inexactness minimizer: \\{\scriptsize $\mathbf{w}_i^{t+1} \approx \mathop{\mathrm{argmin}}\limits_{\mathbf{w}} \mathcal{L}(\mathbf{w};\mathcal{D}_i) + \frac{\tilde{\lambda}}{2}\cdot \lVert \mathbf{w} - \tilde{\mathbf{w}}^{{t+1}} \rVert_2^2$}
   \STATE Client sends $\mathbf{w}_i^{t+1}$ to server
   \ENDFOR
   \STATE Server aggregates $\{\mathbf{w}_i^{t+1}\}_{i=1}^N$ to get $\mathbf{w}^{t+1}$
    \ENDFOR
\end{algorithmic}
\end{algorithm}

\section{Theoretical Analysis}

In this section, we provide the theoretical convergence analysis for the proposed FedPTR framework. We first state several assumptions needed for the convergence analysis.

\vspace{2pt}
\begin{assumption}\label{as:smooth}
The loss function on the $i$-th client $\cL(\wb; \cD_i)$ is $L$-smooth, i.e., $\forall \wb,\vb \in \RR^d$, $\big|\cL(\wb; \cD_i)-\cL(\vb; \cD_i)-\langle\nabla \cL(\vb; \cD_i), \wb-\vb\rangle\big| \leq \frac{L}{2}\|\wb-\vb\|_2^2$. This also implies the $L$-gradient Lipschitz condition: $\|\nabla \cL(\vb; \cD_i) - \nabla \cL(\wb; \cD_i)\|_2 \leq L\|\wb-\vb\|_2$. 
\end{assumption}

\vspace{2pt}
\begin{assumption}\label{as:bounded-B}
The gradient of the local loss function on the $i$-th client 
$\nabla \cL(\wb; \cD_i)$ is $B$-locally dissimilar, i.e., $\frac{1}{N} \sum_{i \in [N]} \|\nabla \cL(\wb; \cD_i)\|_2^2 \leq B^2 \|\nabla \cL(\wb)\|_2^2$. 
\end{assumption}
\vspace{2pt}
Assumption \ref{as:smooth} and \ref{as:bounded-B} are standard assumptions in federated learning optimization problems \cite{acar2021federated, li2020federatedprox, karimireddy2020scaffold, xu2021fedcm,gao2022feddc}. 
Assumption \ref{as:bounded-B} describes the dissimilarity between the gradient of local objectives and the global objective.
Note that the commonly used \cite{li2020federatedprox, reddi2020adaptive, wang2022communication}
bounded variance assumption,  $\frac{1}{N} \sum_{i \in [N]}  \|\nabla \cL(\wb; \cD_i) - \nabla \cL(\wb)\|^2 \leq \sigma^2$, also satisfies the condition in Assumption \ref{as:bounded-B}.

\vspace{2pt}
\begin{assumption}\label{as:bounded-g}
The auxiliary loss function on the $i$-th client $\tilde{\cL}(\wb; \tilde{\cD}_i)$ has $\ell_2$-bounded gradient dissimilarity with global objective $\cL(\wb)$, i.e., $\forall \wb \in \RR^d$, it satisfies $\|\nabla \cL(\wb) - \nabla \tilde{\cL}(\wb; \tilde{\cD}_i)\|_2^2 \leq \sigma_d^2.$ 
\end{assumption}
\vspace{2pt}
Assumption \ref{as:bounded-g} measures the gradient dissimilarity between the global loss function and the training loss on the auxiliary dataset $\tilde{\cD}_i$, which implies that the MTT algorithm is actually well functioning: the extracted auxiliary dataset indeed mimics the actual global model training trajectories.

In the following, we will state the theoretical convergence analysis for our proposed FedPTR. For expository purposes, we assume $K = 1$, i.e., performing 1 step of auxiliary model update in each round, though our convergence analysis can also be extended to cases when $K>1$. 
\begin{theorem}[Non-convex FedPTR convergence] \label{thm:convergence}
Under Assumptions \ref{as:smooth}-\ref{as:bounded-g}, assume the local objective functions $\cL(\wb;\cD_i)$ are non-convex, and there exists $L_- >0$, such that $\nabla \cL(\wb;\cD_i) \succeq -L_- \mathbf{I}$ with $\mu= \lambda - L_- >0$. If $\eta \leq \frac{1}{2\lambda}$, $\gamma$ is selected to be sufficient small, $\mu$ and $\lambda$ satisfies
\begin{align}
    & \frac{1}{6\lambda} - \frac{4L^2}{\lambda \mu^2} - \frac{2L^2 }{ \lambda \mu^2} - \frac{1}{\lambda N} - \frac{2L}{\mu^2} - \frac{4L }{\mu^2} \geq 0, 
\end{align}
then the iterations of FedPTR satisfy 
\begin{small}
    
\begin{equation}\label{eq:5.2}
\begin{split}
    \min_{t\in [T]} \EE[\|\nabla \cL(\wb_t)\|_2^2] \leq &\cO\bigg(\frac{\EE[\cL(\wb_{0})] - \EE[\cL(\wb_{T})]}{6\lambda T}\bigg) \\
    & \quad + \cO(\eta^2 \sigma_d^2).
\end{split}
\end{equation}
\end{small}
\end{theorem}
 The proof of Theorem \ref{thm:convergence} is provided in Appendix \ref{sec:proof}. 
 Theorem \ref{thm:convergence} suggests that the gradient norm of FedPTR has an upper bound that contains two terms. The first term vanishes as the communication round $T$ increases and the second term is related to $\sigma_d^2$, i.e., the gradient dissimilarity between global training loss and the training loss on the auxiliary dataset. This implies that if the auxiliary gradient differs significantly from the global gradient, it will hurt FedPTR's performance as larger $\sigma_d$ leads to worse convergence. 

In such a case, the auxiliary gradient cannot well-represent the global gradient direction, hence to obtain a desired convergence performance, it requires a smaller $\eta$ to control the auxiliary model $\tilde{\wb}_{t+1}$ to stay close to the global initial model $\wb_t$ and obtain smaller second term in Eq \eqref{eq:5.2}. 

\begin{remark}
Theorem \ref{thm:convergence} could be extended to our proposed server-side FedPTR (Algorithm \ref{FedPTR_server_distill}) with slighly modifying Assumption \ref{as:bounded-g} to $\|\nabla \cL(\wb) - \nabla \tilde{\cL}(\wb; \tilde{\cD})\|_2^2 \leq \sigma_d^2$, which describes the dissimilarity between the server-side auxiliary loss and the global objective function. 
\end{remark}

\begin{corollary}\label{cor:conv-rate}
If choosing the auxiliary learning rate as $\eta = \Theta (1/\sqrt{T})$, then the convergence rate for Algorithm \ref{FedPTR} satisfies $\min_{t \in [T]} \EE[\|\nabla \cL(\wb^t) \|_2^2] = \cO (1/T)$.
\end{corollary}
Corollary \ref{cor:conv-rate} suggests that with an appropriately chosen learning rate, FedPTR achieves a convergence rate of $\cO(1/T)$ in stochastic non-convex optimization settings, which matches the convergence rate of federated learning with proximal term such as FedProx \cite{li2020federatedprox} and FedDyn \cite{acar2021federated}.

\section{Experiments} \label{experiments}
In this section, we present comprehensive experimental results to show the effectiveness of FedPTR when competing with other state-of-the-art federated learning methods under various heterogeneous data partitions. We first introduce our experimental settings in \cref{experimental settings} and give a detailed analysis of our experimental results and ablation study in the following sections. We report the average accuracy of the last five global rounds on the test set.
\subsection{Experimental Settings} \label{experimental settings}

\noindent\textbf{Datasets.} We conduct experiments on four benchmark datasets: FashionMNIST \cite{xiao2017fashion}, CIFAR10 \cite{krizhevsky2009learning} CIFAR100 \cite{krizhevsky2009learning}, and TinyImageNet \cite{le2015tiny}. We follow the commonly used mechanism to perform data partition through Dirichlet distribution \cite{wang2020federated}. The parameter $\alpha$ used in Dirichlet sampling determines the non-i.i.d degree. with smaller $\alpha$, non-i.i.d degree becomes higher. In our experiments, we focus on the high-level data heterogeneity with $\alpha = 0.01$ and $\alpha = 0.04$.

\noindent\textbf{Baselines.} We compare our method FedPTR (and FedPTR-S) with five state-of-the-art federated learning methods: FedAvg \cite{mcmahan2017communication}, FedProx \cite{li2020federatedprox}, SCAFFOLD \cite{karimireddy2020scaffold}, FedDyn \cite{acar2021federated}, and FedDC \cite{gao2022feddc}. All the baseline methods except FedAvg contain utilization of the global information especially on the consistency of the global model and local model while FedPTR integrates the global information by training an auxiliary dataset based on the global parameters.

\noindent\textbf{Hyper-parameters and Model Architectures.}  In our experiments, we set the batch size as 500 and the local training epoch as 1. By default, we use SGD with a learning rate $\eta =0.01$ and momentum of $0.5$ as the optimizer. We use 10 images per class for FashionMNIST/CIFAR10 and 2 images per class for CIFAR100.  We set trajectory projection steps $K=5$ and the base $\lambda$ for layer-adaptive regularization parameter as $\lambda = 0.05$. We also tune the hyperparameters in all other baselines for their best performances.
In matching training trajectories part (\cref{MTT}), following \cite{cazenavette2022dataset}, we utilize a learnable $\beta$ with initialized $\beta = 0.01$, trajectory matching outer steps $H = 20$, inner steps $R = 10$, while the rest of the MTT hyper-parameters are set as the default value in \cite{cazenavette2022dataset}. In terms of the model architecture, we follow \cite{cazenavette2022dataset} and adopt ConvNet\cite{gidaris2018dynamic} as the default CNN architecture. 
\begin{table*}[ht!]
    \centering
    \resizebox{0.9\textwidth}{!}{
    \begin{tabular}{|c | c | c c c c c c c|}
    \hline
    Participation Ratio & Dataset &  FedAvg & FedProx & Scaffold & FedDyn & FedDC & FedPTR & FedPTR-S \\\hline\hline
    \multirow{8}{*}{$\boldsymbol{100\%}$}  &\multicolumn{8}{c|}{\textbf{Dirichlet(0.01)}}\\\cline{2-9}
     & FMNIST & 83.69 & 83.75  & 82.21  & 84.36  & 78.23  &$\underline{85.11}$  & $\mathbf{85.48}$\\ \cline{2-9}
     & CIFAR10 & 55.43 & 55.94  & 56.58  & 62.29  & 52.55  & $\mathbf{69.04}$ & $\underline{68.16}$\\ \cline{2-9}
     & CIFAR100 & 37.32 & 37.33  & 42.02  & 41.73  & 39.52  & $\mathbf{44.98}$ & $\underline{44.82}$\\ \cline{2-9}

     & \multicolumn{8}{c|}{\textbf{ Dirichlet(0.04)}}\\\cline{2-9}
     & FMNIST & 85.05 & 84.94  & 84.59  & 85.26  & 82.17   & $\mathbf{86.80}$ & $\underline{85.90}$\\ \cline{2-9}
     & CIFAR10 & 65.44 & 65.46  & 68.98  & 70.26  & 68.93  & $\underline{72.78}$ & $\mathbf{73.07}$\\ \cline{2-9}
     & CIFAR100 & 39.65 & 39.63  & 44.66  & 45.57  & 44.23  & $\underline{45.87}$ & $\mathbf{46.10}$\\ \cline{1-9}

     \multirow{8}{*}{$\boldsymbol{50\%}$} &\multicolumn{8}{c|}{\textbf{Dirichlet(0.01)}}\\\cline{2-9}
     & FMNIST & 81.02 & 81.05  & 75.57  & 79.39  & 73.18  & $\mathbf{81.76}$ & \underline{81.71} \\ \cline{2-9}
     & CIFAR10 & 49.90 & 50.28  & 54.91  & $\underline{56.81}$  & 50.65  & $\mathbf{56.95}$ & 54.24 \\ \cline{2-9}
     & CIFAR100 & 37.07 & 37.04  & 37.35  & 37.61  & 37.98  & $\mathbf{40.11}$  & $\underline{39.49}$ \\ \cline{2-9}

     & \multicolumn{8}{c|}{\textbf{ Dirichlet(0.04)}}\\\cline{2-9}
     & FMNIST & 81.75 & 81.73  & 77.42  & 80.71  & 77.44  & $\underline{82.85}$ & $\mathbf{83.05}$ \\ \cline{2-9}
     & CIFAR10 & 56.84 & 56.80  & 57.58 & 59.45  & 61.31 & $\mathbf{65.60}$ & $\underline{65.06}$ \\ \cline{2-9}
     & CIFAR100 & 38.22 & 38.24  & 38.33  & 40.68  & 40.50  & $\underline{43.63}$ & $\mathbf{44.99}$ \\ \cline{1-9}
     \hline
     \multirow{8}{*}{$\boldsymbol{30\%}$} &\multicolumn{8}{c|}{\textbf{Dirichlet(0.01)}}\\\cline{2-9}
     & FMNIST & 75.89 & 76.40 & 73.05  & 75.09 & 73.09  & $\mathbf{76.79}$ & $\underline{76.78}$ \\ \cline{2-9}
     & CIFAR10 & 41.82 & 43.06  & $\underline{47.92}$  & 45.20 & 44.01 & $\mathbf{48.94}$ & 47.41 \\ \cline{2-9}
     & CIFAR100 & 29.68 & 29.65  & 29.18  & 30.80 & 32.67  & $\mathbf{35.41}$ & $\underline{33.64}$ \\ \cline{2-9}

     & \multicolumn{8}{c|}{\textbf{ Dirichlet(0.04)}}\\\cline{2-9}
     & FMNIST & 77.79 & 78.75  & 77.11  & 78.89 & 76.99  & $\underline{79.03}$ & $\mathbf{79.48}$\\ \cline{2-9}
     & CIFAR10 & 48.50 & 48.52  & 53.14  & 44.21 & 55.72  & $\underline{56.67}$ & $\mathbf{57.75}$\\ \cline{2-9}
     & CIFAR100 & 32.95 & 32.87 &  37.46  & 36.82 & $\underline{37.57}$ & 37.49 & $\mathbf{38.21}$ \\ 
    \hline
    \end{tabular}}

\vskip -0.1in
    \caption{Comparison of test accuracy for FedPTR (FedPTR-S) and other baselines in training different datasets and under different participation ratios with 10 clients. \textbf{Bold} and \underline{underline} represent the best and second best results.}
    \label{tab:main}
\vskip -0.1in
\end{table*}

\vspace{-5pt}
\begin{table*}[ht]
    \centering
    \resizebox{0.9\textwidth}{!}{
    \begin{tabular}{|c | c | c c c c c c c|}
    \hline
    Setting & Dataset &  FedAvg & FedProx & Scaffold & FedDyn & FedDC & FedPTR & FedPTR-S \\\hline\hline
    \multirow{6}{*}{\begin{tabular}[c]{@{}c@{}}40 clients with \\  $\boldsymbol{25\%}$ participation ratio\end{tabular}}  &\multicolumn{8}{c|}{\textbf{Dirichlet(0.01)}}\\\cline{2-9}
     & CIFAR100 & 32.68 & 32.71  & 36.07  & 29.91  & 35.66  & $\mathbf{37.50}$ & $\underline{36.40}$\\ \cline{2-9}

     & \multicolumn{8}{c|}{\textbf{ Dirichlet(0.04)}}\\\cline{2-9}
     & FMNIST& 83.73 & 83.77 & 79.76 & 78.78 & 82.71 & $\mathbf{85.92}$ & $\underline{84.68}$ \\ \cline{2-9}
     & CIFAR10 & 58.00 & 58.06 & 58.32 & 54.11 & 55.13 & $\mathbf{60.30}$ & $\underline{59.73}$\\ \cline{2-9}
     & CIFAR100 & 38.35 & 38.32 & 40.34 & 38.67 & $\mathbf{42.74}$ & $\underline{41.90}$ & 41.11 \\ 
     \hline
    \end{tabular}}
    \vskip -0.1in
    \caption{Comparions of FedPTR (FedPTR-S) with other baselines for partial participation in training different datasets with $40$ clients. \textbf{Bold} and \underline{underline} represent the best and second best results.}
    \label{tab:more_client}
\vskip -0.15in
\end{table*}
\subsection{Comparison of Different Federated Learning Methods on Heterogeneous Data}
We compare our method FedPTR (and FedPTR-S) with five state-of-the-art federated learning optimization methods: FedAvg \cite{mcmahan2017communication}, FedProx \cite{li2020federatedprox}, SCAFFOLD \cite{karimireddy2020scaffold}, FedDyn \cite{acar2021federated}, and FedDC \cite{gao2022feddc}, on varying data heterogeneity, the number of clients, and participation rate. 
\cref{tab:main} shows the comparison results of our method FedPTR with baselines for the scenario of full participation and partial participation with 10 clients. We perform evaluations on two different heterogeneous data partitions, $\text{Dirichlet } \alpha = 0.04$ and $\text{Dirichlet } \alpha = 0.01$. In the full participation setting (participation ratio $100\%$), we observe that our method outperforms previous methods regarding test accuracy on all datasets. In particular, we achieve the absolute improvement of $6.75\%$ compared with the best-performing baseline FedDyn and over $12\%$ improvement over other baselines on CIFAR10 dataset with $\text{Dirichlet } \alpha = 0.01$. 
Moreover, it can be seen that previous methods suffer from performance decreases of up to $16.38\%$ when the data partition becomes more heterogeneous (i.e., from $\text{Dirichlet } \alpha = 0.04$ to $\text{Dirichlet } \alpha = 0.01$),
while our method only reduces the test accuracy by up to $4.91\%$. Such results demonstrate that our method is consistently robust against extremely heterogeneous training data. In partial participation, we evaluate our method and baselines with $50\%$ and $30\%$ participation ratios. We can see that our method obtains relative improvements of $0.9\% \sim 17.7\%$ and $1.2\% \sim 19.3\%$ on $50\%$ participation and $30\%$ participation scenarios compared with FedAvg. These results clearly show the effectiveness of our method for the scenario of partial participation. To further verify our method in the scenario of more clients. We additionally compare FedPTR with baselines in 40 clients and $25\%$ participation ratio as shown in \cref{tab:more_client}. We observe that our method FedPTR (and FedPTR-S) still maintains excellent performance compared with the previous methods. Note that in the setting of Dirichlet $\alpha = 0.01$, we only evaluate our method on CIFAR100 since FashionMNIST and CIFAR10 with fewer categories cannot make each client have training data. We defer the experiments on TinyImageNet in Appendix \ref{ap:tiny}.

\subsection{Reducing the Updates of Auxiliary Dataset}
The original design of FedPTR (FedPTR-S) requires the clients/server to update the auxiliary datasets in each communication round. In this section, we explore how performance scales if we reduce the frequency of the auxiliary dataset updates. In \cref{fig:reduce_updates}, we compare the test accuracy against the communication rounds on CIFAR10 $\text{Dirichlet } \alpha = 0.01$ by performing one MTT update per 1, 4, 10, and 40 rounds. We additionally try a one-shot MTT strategy by only updating the auxiliary dataset once. 
Note that when there is no MTT update over the auxiliary dataset, we still need to perform the trajectory projection step (with the ``old'' auxiliary dataset) for local update regularization. We find that our proposed FedPTR method does not suffer much in performance even when the number of updates of the auxiliary dataset is reduced by $75\%$. Moreover, even the most efficient FedPTR with one-shot MTT can still outperform the other baseline methods such as FedDyn \cite{acar2021federated} on CIFAR10. See more results on FedPTR-S in Appendix \ref{ap:more_updates}.
\begin{figure}[htb]
\vskip -0.1in
\begin{center}
\centerline{\includegraphics[width=0.6\columnwidth]{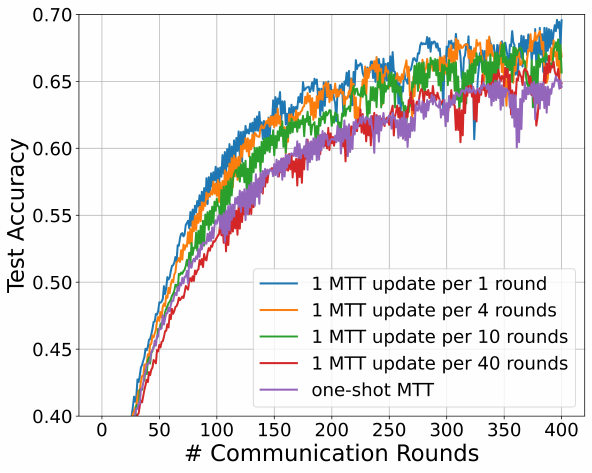}}
\caption{\small Different frequency of MTT updates on CIFAR10 with $\text{Dirichlet } \alpha = 0.01$.}
\label{fig:reduce_updates}
\end{center}
\vskip -0.4in
\end{figure}

\subsection{Changing the Size of Auxiliary Dataset}
We study how the size of the auxiliary dataset affects our final performances. For this, we conduct experiments comparing FedPTR using different sizes of the auxiliary dataset on CIFAR10 and CIFAR100 dataset and show the results in \cref{table:size_cifar10}.
For CIFAR10, we explore 5 different sizes, and the results suggest that a moderate size of the auxiliary dataset leads to the best model training performances.  
When the size is below $100$, the test accuracy is largely affected due to less representation power limited by the auxiliary dataset size. 
On the other hand, when the size grows beyond $100$, there is a clear marginal effect and the influence of the size is not as obvious as before (may even lose performances due to insufficient optimization of the large auxiliary data). For CIFAR100 data, a similar trend can also be observed. Note that using a larger auxiliary dataset leads to a larger computational cost when performing MTT updates in each communication round. In practice, we would be picking a moderate auxiliary dataset size (such as $100$) for the best of both performance and computational cost.

\begin{table}[htb]
\centering
\resizebox{0.6\linewidth}{!}{
    \begin{tabular}{ ccc }
    \toprule
     Dataset & Size of $\mathcal{\tilde{D}}$ & Test Acc \\
     \midrule
     \multirow{5}{*}{CIFAR10}&10 & $65.83$ \\
     & 50 &  $67.02$  \\ 
     & 100 & $69.04$ \\
     & 150 & $68.46$ \\
     & 200 &  $69.00$\\
    \midrule
    \multirow{3}{*}{CIFAR100}&100 & $44.57$ \\
    & 200 & $44.98$\\
    & 300 &  $44.67$\\
    \bottomrule
    \end{tabular}}
\vskip -0.1in
\caption{Comparions of FedPTR with different sizes of the auxiliary dataset with $\text{Dirichlet } \alpha = 0.01$.}
\label{table:size_cifar10}
\vskip -0.2in
\end{table}

\subsection{Quality of Extracted Global Information}
In this subsection, we give a brief analysis on the quality of the extracted global information by the auxiliary dataset. To do so, we measure the similarity between the gradient of the training loss on the auxiliary dataset (denoted by the auxiliary gradient) and the global gradient.
A high similarity between the two gradients could indicate that extracted auxiliary datasets well-approximate the recent learning dynamics and can well project the next-step training trajectory. For computational simplicity, here we directly use the global update direction on the server to approximate the global gradient direction, i.e., $\mathbf{w}^t - \mathbf{w}^{t+1}$. And we use $\mathbf{w}^t - \tilde{\mathbf{w}}^{t+1}_i$ to approximate the auxiliary gradient. To further demonstrate the quality of extracted global information, we also compute the cosine similarity of local gradient $ \mathbf{w}^t - \mathbf{w}^{t+1}_i$ and global gradient. Note that we compute two cosine similarities on vanilla FedAvg, i.e., clients keep the update of the auxiliary dataset and projected training trajectory while  not using any regularization term to avoid the auxiliary dataset influencing the training gradient. We present our results for one randomly chosen client in training a CIFAR100 model with Dirichlet $\alpha=0.01$ in \cref{fig:gradient}. Compared with the local gradient, the auxiliary gradient has a much higher similarity to the global gradient, which suggests that FedPTR can indeed excavate global information from the model training trajectory. Note that the cosine similarity of the auxiliary gradient and global gradient decreases with the increase of the communication rounds. This suggests that it is harder to extract global information at the later phase of model training. Nevertheless, we can still find that the auxiliary gradient shares greater similarity to the global gradient direction compared with the local gradient along the entire training trajectory.
\begin{figure}[htb]
\vskip -0.1in
\begin{center}
\centerline{\includegraphics[width=0.6\columnwidth]{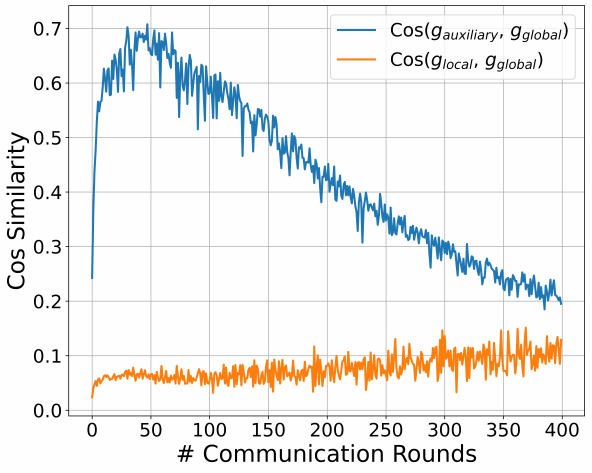}}
\vskip -0.1in
\caption{Comparisons of the gradient cosine similarity. }
\label{fig:gradient}
\end{center}
\vskip -0.35in
\end{figure}

\subsection{Ablation Study}
One of the important hyperparameters of FedPTR is the weight of the regularization term $\lambda$. To test its sensitivity, we consider the non-i.i.d. setting of full participation on CIFAR10 with $\text{Dirichlet } \alpha = 0.01$. We conduct the sensitive test from the candidate set $\{0.01, 0.05, 0.1, 0.2, 1.0\}$. \cref{fig:alpha} shows the convergence plot at varying $\lambda$, while we keep all other hyperparameters of FedPTR consistent. We can observe that the best test accuracy is obtained when $\lambda = 0.1$. The small $\lambda$ (e.g., 0.01) may not have any effect, and the large $\lambda$ (e.g., 1.0) possibly force $\mathbf{w}$ to be close to $\tilde{\mathbf{w}}_i^{{K}}$ thus causing slow convergence. We defer more ablation studies on the trajectory matching/projecting steps, layer-adaptive regularization, and different initialization strategies in Appendix \ref{ap:ab}.

\begin{figure}[htb]
\vskip -0.1in
\begin{center}
\centerline{\includegraphics[width=0.6\columnwidth]{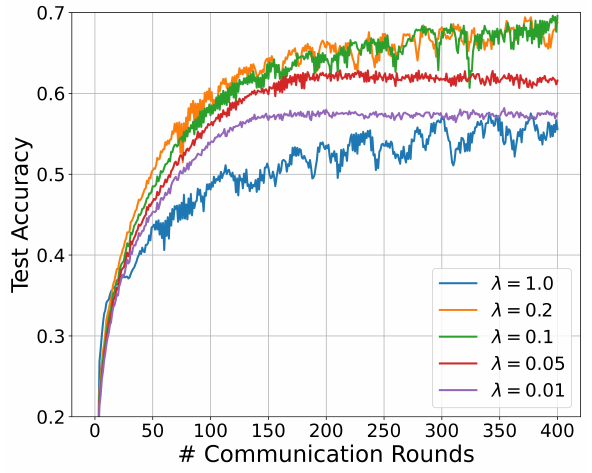}}
\vskip -0.1in
\caption{Sensitivity of $\lambda$ in FedPTR.}
\label{fig:alpha}
\end{center}
\vskip -0.35in
\end{figure}

\section{Conclusion}
Data heterogeneity issue is a critical factor for the practical deployment of Federated Learning algorithms. In this work, we propose a novel federated learning framework FedPTR that handles extreme data heterogeneity across local clients. The proposed method allows each client or the server to extract essential global information from the observed global model training trajectory and utilize it to project the next-step model trajectory for local training regularization. The exhaustive experimental results on various benchmark datasets and different non-i.i.d settings clearly demonstrate the effectiveness of our proposed FedPTR in training real-world machine learning models. One potential limitation of FedPTR is that it cannot be directly applied to language tasks since the application of MTT on language tasks is still an open question. We will leave it as our feature work.

\bibliography{anthology,custom}
\bibliographystyle{acl_natbib}

\appendix

\onecolumn
\section{Proof of Theorem \ref{thm:convergence}.} \label{sec:proof}
For notation convenience, we conclude $\cL(\wb; \cD_i)$ as $\cL_i(\wb)$ and $L(\wb; \tilde{\cD}_i)$ as $\tilde{\cL}_i(\wb)$ in the following.
\begin{lemma}\label{lm:hi-smooth}
    For $h_i(\wb) := \cL_i(\wb) + \frac{\lambda}{2}\|\wb - \wb^t + \eta \nabla \Tilde{\cL}_i (\wb^t) \|^2$, there is $\|\nabla h_i(\wb) - \nabla h_i(\vb) \|\leq (L+\lambda) \|\wb - \vb\|, \forall \wb, \vb \in \RR^d $.
\end{lemma}

\begin{proof}[Proof of Lemma \ref{lm:hi-smooth}]
\begin{align}
    & \|\nabla h_i(\wb) - \nabla h_i(\yb) \| \notag\\
    & = \|\nabla \cL_i(\wb) - \nabla \cL_i(\yb) + \lambda \wb - \lambda \yb\| \notag\\
    & \leq \|\nabla \cL_i(\wb) - \nabla \cL_i(\yb)\| + \lambda \|\wb-\yb\| \notag\\
    & \leq (L+\lambda) \|\wb-\yb\|.
\end{align}
This concludes the proof.
\end{proof}

\begin{proof}[Proof of Theorem \ref{thm:convergence}]
    
Recall the notation and the objective function, we have the following for the steps,   
\begin{align}
    \wb_i^{t,k+1} 
    & = \wb_i^{t,k} - \eta (\gb_i^{t,K} + \lambda (\wb_i^{t,k} - \tilde{\mathbf{w}}_i^{t+1})) \notag\\
    & = \wb_i^{t,k} - \eta (\gb_i^{t,K} + \lambda (\wb_i^{t,k} - \wb_{t} + \eta \nabla \tilde{\cL}_i (\wb^t))).
\end{align}
We also have 
\begin{align}
    h_i(\wb) & = f_i(\wb) + \frac{\lambda}{2}\|\wb - \wb^t + \eta \nabla \tilde{\cL}_i (\wb^t) \|^2, \notag\\
    \nabla h_i(\wb) & = \nabla \cL_i(\wb) + \lambda (\wb - \wb_{t} + \eta \nabla \tilde{\cL}_i (\wb^t)).
\end{align}

Denote $\eb_i^{t,k+1}$ as follows
\begin{align}
    \eb_i^{t,k+1} & = \nabla \cL_i(\wb_i^{t,k}) + \lambda (\wb_i^{t,k} - \wb_{t} + \eta \nabla \tilde{\cL}_i (\wb^t)) = \nabla h_i(\wb_i^{t,k}) , 
\end{align}
by the notation of $\gamma$-inexactness, there is 
\begin{align}
    \|\eb_i^{t,k+1}\|^2 & \leq \gamma^2 \|\nabla \cL_i(\wb^t - \eta\nabla \tilde{\cL}_i (\wb^t)) \|^2 \notag\\
    & = \gamma^2 \|\nabla \cL_i(\tilde{\mathbf{w}}_i^{t+1} ) \|^2 \notag\\
    & = \gamma^2 \|\nabla \cL_i(\tilde{\mathbf{w}}_i^{t+1}) - \nabla \cL_i(\wb^t) + \nabla \cL_i(\wb^t) \|^2 \notag\\
    & \leq 2 \gamma^2 \|\nabla \cL_i(\wb^t) \|^2 + 2\gamma^2 L^2 \|\tilde{\mathbf{w}}_i^{t+1} - \wb^t\|^2 \notag\\
    & = 2\gamma^2 \|\nabla \cL_i(\wb^t) \|^2 + 2\gamma^2 L^2 \eta^2\|\nabla \tilde{\cL}_i (\wb^t) \|^2,
\end{align}
where the first equation holds by the definition of distilled model weight $\tilde{\mathbf{w}}_i^{t+1}$, the second inequality holds by the $L$-smoothness of $f_i$, and the last equation is from the definition of $\tilde{\mathbf{w}}_i^{t+1}$. 
Then we have 
\begin{align}
    \EE_t[\wb^{t+1} - \wb^t] & = \frac{1}{N} \sum_{i \in [n]} \EE_t[\wb_i^{t,k} - \wb^t] \notag\\
    & = -\frac{1}{\lambda} \frac{1}{N} \sum_{i \in [n]} \EE_t[\gb_i^{t,K} + \lambda \eta \nabla \tilde{\cL}_i (\wb^t) - \eb_i^{t,K} ].
\end{align}
Define $\mu = \lambda - L_- >0$ [$L_-$] and $\wb^{t+1}_{i,*} = \argmin_{\wb} h_i(\wb; \wb^t)$, then we have 
\begin{align}
    \|\wb^{t+1}_{i,*} - \wb_i^{t,k}\| & \leq \frac{1}{\mu} \|\nabla h_i (\wb^{t+1}_{i,*}) - \nabla h_i (\wb_i^{t,k}) \| \notag\\
    & = \frac{1}{\mu} \|\nabla h_i (\wb_i^{t,k}) \| \notag\\
    & \leq \frac{\gamma}{\mu} \|\nabla h_i (\tilde{\mathbf{w}}_i^{t+1}) \| \notag\\
    & \leq \frac{\gamma}{\mu} \|\nabla h_i (\tilde{\mathbf{w}}_i^{t+1}) - \nabla h_i (\wb^t)\| + \frac{\gamma}{\mu} \|\nabla h_i (\wb^t) \| \notag\\
    & \leq \frac{\gamma}{\mu} (L+\lambda) \|\tilde{\mathbf{w}}_i^{t+1} - \wb^t\| + \frac{\gamma}{\mu} \|\nabla \cL_i (\wb^t) + \lambda \eta \nabla \tilde{\cL}_i (\wb^t) \| \notag\\
    & = \frac{\gamma \eta}{\mu} (L+\lambda) \|\nabla \tilde{\cL}_i (\wb^t)\| + \frac{\gamma}{\mu} \|\nabla \cL_i (\wb^t) + \lambda \eta \nabla \tilde{\cL}_i (\wb^t) \|,
\end{align}
where the first inequality holds by the $\mu$-strong convexity of $h_i$, the first equation holds by $\nabla h_i (\wb^{t+1}_{i,*}) = 0$ since $\wb^{t+1}_{i,*} = \argmin_{\wb} h_i(\wb; \wb^t)$, and the second inequality holds by $\gamma$-inexactness. The third inequality follows the triangle inequality, the forth one follows Lemma \ref{lm:hi-smooth}. 
Similarly by the $\mu$-strongly convexity, we have 
\begin{align}
    \|\wb^{t+1}_{i,*} - \wb^t\| \leq \frac{1}{\mu} \|\nabla h_i(\wb^t)\| = \frac{1}{\mu} \|\nabla \cL_i(\wb^t) + \lambda \eta \nabla \tilde{\cL}_i (\wb^t)\|.
\end{align}
By the triangle inequality, we have 
\begin{align}
    \|\wb_i^{t,k} - \wb^t \| & \leq \|\wb^{t+1}_{i,*} - \wb^t\| + \|\wb^{t+1}_{i,*} - \wb_i^{t,k}\| \notag\\
    & \leq \frac{1+\gamma}{\mu} \|\nabla \cL_i(\wb^t) + \lambda \eta \nabla \tilde{\cL}_i (\wb^t)\| + \frac{\gamma \eta (L+\lambda)}{\mu} \|\nabla \tilde{\cL}_i (\wb^t)\|
\end{align}
and 
\begin{align}
    \|\wb_i^{t,k} - \wb^t \|^2 & \leq 2\|\wb^{t+1}_{i,*} - \wb^t\|^2 + 2\|\wb^{t+1}_{i,*} - \wb_i^{t,k}\|^2 \notag\\
    & \leq \frac{2+4\gamma^2}{\mu^2} \|\nabla \cL_i(\wb^t) + \lambda \eta \nabla \tilde{\cL}_i (\wb^t)\|^2 + \frac{4\gamma^2 \eta^2 (L+\lambda)^2}{\mu^2} \|\nabla \tilde{\cL}_i (\wb^t)\|^2.
\end{align}
Since the objective function $f$ is $L$ smooth, taking conditional expectation at time $t$, (abbreviate $\EE_t$ as $\EE$) we have 
\begin{align}
    \EE_t[\cL(\wb^{t+1})] & \leq \cL(\wb_{t}) + \EE_t\langle \nabla \cL(\wb_{t}), \wb^{t+1} - \wb^t \rangle + \frac{L}{2} \EE_t[\|\wb^{t+1} - \wb^t\|^2] \notag\\
    & \leq \cL(\wb_{t}) + \underbrace{\EE\bigg\langle \nabla \cL(\wb_{t}), -\frac{1}{\lambda} \frac{1}{N} \sum_{i \in [n]} [\gb_i^{t,K} + \lambda \eta \nabla \tilde{\cL}_i (\wb^t) - \eb_i^{t,K}]\bigg\rangle}_{I_1} + \underbrace{\frac{L}{2} \EE[\|\wb^{t+1} - \wb^t\|^2]}_{I_2}.
\end{align}
For term $I_2$, we have 
\begin{align}
    I_2 & = \frac{L}{2} \EE_t[\|\wb^{t+1} - \wb^t\|^2] \notag\\
    & = \frac{L}{2}\EE\bigg[\bigg\|\frac{1}{N} \sum_{i \in [n]} \wb_i^{t,k} - \wb^t \bigg\|^2\bigg] \notag\\
    & \leq \frac{L}{2}\frac{1}{N} \sum_{i \in [n]} \EE[\|\wb_i^{t,k} - \wb^t \|^2] \notag\\
    & \leq \frac{1+2\gamma^2}{\mu^2} \frac{L}{n} \sum_{i \in [n]} \EE[\|\nabla \cL_i(\wb^t) + \lambda \eta \nabla \tilde{\cL}_i (\wb^t)\|^2] + \frac{4\gamma^2 \eta^2 L^2}{\mu^2} \frac{L}{n} \sum_{i \in [n]} \EE[\|\nabla \tilde{\cL}_i (\wb^t)\|^2 ]\notag\\
    & \leq \frac{2+4\gamma^2}{\mu^2} \frac{L}{n} \sum_{i \in [n]} \EE[\|\nabla \cL_i(\wb^t) \|^2] + \bigg(\frac{4\gamma^2 \eta^2 (L+\lambda)^2}{\mu^2} + \frac{(2+4\gamma^2)\lambda^2 \eta^2}{\mu^2}\bigg) \frac{L}{n} \sum_{i \in [n]} \EE[\| \nabla \tilde{\cL}_i (\wb^t)\|^2] \notag\\
    & \leq \frac{LB^2(2+4\gamma^2)}{\mu^2} \EE[\|\nabla \cL_i(\wb^t)\|^2] + \bigg(\frac{4\gamma^2 \eta^2 (L+\lambda)^2}{\mu^2} + \frac{(2+4\gamma^2)\lambda^2 \eta^2}{\mu^2}\bigg)\frac{L}{n} \sum_{i \in [n]} \EE[\| \nabla \tilde{\cL}_i (\wb^t) \mp \nabla\cL(\wb^t)\|^2],
\end{align}
where the first and third inequalities hold by Cauchy inequality, the forth one holds by Assumption \ref{as:bounded-B}, and the last one holds by Assumption \ref{as:bounded-g}. Then we have 
\begin{align}
    & \bigg(\frac{4\gamma^2 \eta^2 (L+\lambda)^2}{\mu^2} + \frac{(2+4\gamma^2)\lambda^2 \eta^2}{\mu^2}\bigg)\frac{L}{n} \sum_{i \in [n]} \EE[\| \nabla \tilde{\cL}_i (\wb^t) \mp \nabla\cL(\wb^t) \|^2] \notag\\
    & \leq 2 \bigg(\frac{4\gamma^2 \eta^2 (L+\lambda)^2}{\mu^2} + \frac{(2+4\gamma^2)\lambda^2 \eta^2}{\mu^2}\bigg)\frac{L}{n} \sum_{i \in [n]} \EE[\| \nabla \tilde{\cL}_i (\wb^t) -  \nabla\cL(\wb^t) \|^2] \notag\\
    & \quad + 2 \bigg(\frac{4\gamma^2 \eta^2 (L+\lambda)^2}{\mu^2} + \frac{(2+4\gamma^2)\lambda^2 \eta^2}{\mu^2}\bigg)\frac{L}{n} \sum_{i \in [n]} \EE[\|\nabla\cL(\wb^t) \|^2] \notag\\
    & \leq \bigg(\frac{8\gamma^2 \eta^2 (L+\lambda)^2}{\mu^2} + \frac{(4+8\gamma^2)\lambda^2 \eta^2}{\mu^2}\bigg) L \sigma_d^2 + \bigg(\frac{8\gamma^2 \eta^2 (L+\lambda)^2}{\mu^2} + \frac{(4+8\gamma^2)\lambda^2 \eta^2}{\mu^2}\bigg) L \EE[\|\nabla\cL(\wb^t) \|^2].
\end{align}
Thus
\begin{align}
    I_2 & \leq \bigg(\frac{2+4\gamma^2}{\mu^2} + \frac{8\gamma^2 \eta^2 (L+\lambda)^2}{\mu^2} + \frac{(4+8\gamma^2)\lambda^2 \eta^2}{\mu^2}\bigg) L \EE[\|\nabla \cL(\wb^t)\|^2] \notag\\
    & \quad + \bigg(\frac{8\gamma^2 \eta^2 (L+\lambda)^2}{\mu^2} + \frac{(4+8\gamma^2)\lambda^2 \eta^2}{\mu^2}\bigg) L \sigma_d^2.
\end{align}
Similarly, it also implies
\begin{align}
    \EE[\|\wb^{t+1} - \wb^t\|^2] & \leq \frac{1}{N} \sum_{i \in [n]}\EE[\|\wb^{t+1}_i - \wb^t\|^2] \notag\\
    & \leq 2\bigg(\frac{2+4\gamma^2}{\mu^2} + \frac{8\gamma^2 \eta^2 (L+\lambda)^2}{\mu^2} + \frac{(4+8\gamma^2)\lambda^2 \eta^2}{\mu^2}\bigg) \EE[\|\nabla \cL(\wb^t)\|^2] \notag\\
    & \quad + 2\bigg(\frac{8\gamma^2 \eta^2 (L+\lambda)^2}{\mu^2} + \frac{(4+8\gamma^2)\lambda^2 \eta^2}{\mu^2}\bigg) \sigma_d^2.
\end{align}
For term $I_1$, we have 
\begin{align}
    I_1 & = \EE \bigg\langle \nabla \cL(\xb_{t}), -\frac{1}{\lambda} \frac{1}{N} \sum_{i \in [n]} [\gb_i^{t,K} + \lambda \eta \nabla \tilde{\cL}_i (\wb^t) - \eb_i^{t,K}] + \frac{1}{\lambda} \nabla \cL(\wb^t) - \frac{1}{\lambda} \nabla \cL(\wb^t) \bigg\rangle \notag\\
    & = -\frac{1}{\lambda} \EE[\|\nabla \cL(\xb_{t})\|^2] + \EE \bigg\langle \nabla \cL(\xb_{t}), -\frac{1}{\lambda} \frac{1}{N} \sum_{i \in [n]} [ \nabla \cL_i(\wb_i^{t,K}) + \lambda \eta \nabla \tilde{\cL}_i (\wb^t) - \eb_i^{t,K}] + \frac{1}{\lambda} \nabla \cL (\wb^t) \bigg\rangle \notag\\
    & = -\frac{1}{\lambda} \EE[\|\nabla \cL(\xb_{t})\|^2] + \EE\bigg\langle \nabla \cL(\xb_{t}), -\frac{1}{\lambda} \frac{1}{N} \sum_{i \in [n]} [\nabla \cL_i(\wb_i^{t,K}) + \lambda \eta \nabla \tilde{\cL}_i (\wb^t) - \eb_i^{t,K}] + \frac{1}{\lambda} \frac{1}{N} \sum_{i \in [n]} \nabla \cL_i (\wb^t) \bigg\rangle \notag\\
    & = -\frac{1}{\lambda} \EE[\|\nabla \cL(\xb_{t})\|^2] + \frac{1}{2\lambda} \EE[\|\nabla \cL(\xb_{t})\|^2] \notag\\
    & \quad + \underbrace{\frac{1}{2\lambda} \EE\bigg[\bigg\|- \frac{1}{N} \sum_{i \in [n]} [\nabla \cL_i(\wb_i^{t,K}) + \lambda \eta \nabla \tilde{\cL}_i (\wb^t) - \eb_i^{t,K}] + \nabla \cL (\wb^t)\bigg\|^2\bigg]}_{I_3} \notag\\
    & \quad - \frac{1}{2\lambda} \EE\bigg[\bigg\|- \frac{1}{N} \sum_{i \in [n]} [\nabla \cL_i(\wb_i^{t,K}) + \lambda \eta \nabla \tilde{\cL}_i (\wb^t) - \eb_i^{t,K}]\bigg\|^2\bigg],
\end{align}
where the second equation holds due to the unbiasedness of stochastic gradient $\gb_i^{t,K}$, the third equation holds by the definition of global objective function, the forth one holds by the fact of $\langle \ab, \bb \rangle = \frac{1}{2}[\|\ab\|^2 + \|\bb\|^2 - \|\ab - \bb\|^2]$. For $I_3$ term, we have 
\begin{align}
    I_3 & = \frac{1}{2\lambda} \EE\bigg[\bigg\|- \frac{1}{N} \sum_{i \in [n]} \nabla \cL_i(\wb_i^{t,K}) - \frac{1}{N} \sum_{i \in [n]}\lambda \eta \nabla \tilde{\cL}_i (\wb^t) + \frac{1}{N} \sum_{i \in [n]}\eb_i^{t,K} + \frac{1}{N} \sum_{i \in [n]} \nabla \cL_i (\wb^t)\bigg\|^2\bigg] \notag\\
    & \leq \frac{1}{\lambda} \EE\bigg[\bigg\|\frac{1}{N} \sum_{i \in [n]} \nabla \cL_i (\wb^t) - \frac{1}{N} \sum_{i \in [n]} \nabla \cL_i(\wb_i^{t,K}) \bigg\|^2\bigg] + \frac{1}{\lambda} \EE\bigg[\bigg\|- \frac{1}{N} \sum_{i \in [n]}\lambda \eta \nabla \tilde{\cL}_i (\wb^t) + \frac{1}{N} \sum_{i \in [n]}\eb_i^{t,K} \bigg\|^2\bigg] \notag\\
    & \leq \frac{L^2}{\lambda} \frac{1}{N} \sum_{i \in [n]} \EE[\| \wb_i^{t,K} - \wb^t \|^2] + \frac{2}{\lambda} \EE\bigg[\bigg\|\frac{\lambda \eta}{n} \sum_{i \in [n]} \nabla \tilde{\cL}_i (\wb^t) \bigg\|^2 \bigg] + \frac{2}{\lambda} \frac{1}{N} \sum_{i \in [n]} \EE[\|\eb_i^{t,K} \|^2] \notag\\
    & \leq \frac{L^2}{\lambda} \bigg[\bigg(\frac{4+8\gamma^2}{\mu^2} + \frac{16\gamma^2 \eta^2 (L+\lambda)^2}{\mu^2} + \frac{(8+16\gamma^2)\lambda^2 \eta^2}{\mu^2}\bigg) \EE[\|\nabla \cL(\wb^t)\|^2] \notag\\
    & \quad + \bigg(\frac{16\gamma^2 \eta^2 (L+\lambda)^2}{\mu^2} + \frac{(8+16\gamma^2)\lambda^2 \eta^2}{\mu^2}\bigg) \sigma_d^2 \bigg]\notag\\
    & \quad + \frac{2\lambda \eta^2}{n} \sum_{i \in [n]} \EE[\|\nabla \tilde{\cL}_i (\wb^t) \|^2] + \frac{2}{\lambda} \frac{1}{N} \sum_{i \in [n]} [2\gamma^2 \EE[\|\nabla \cL_i(\wb^t)\|^2] + 2 \gamma^2 L^2 \eta^2 \EE[\|\nabla \tilde{\cL}_i (\wb^t)\|^2]] \notag\\
    & \leq \frac{L^2}{\lambda} \bigg[\bigg(\frac{4+8\gamma^2}{\mu^2} + \frac{16\gamma^2 \eta^2 (L+\lambda)^2}{\mu^2} + \frac{(8+16\gamma^2)\lambda^2 \eta^2}{\mu^2}\bigg) \EE[\|\nabla \cL(\wb^t)\|^2] \notag\\
    & \quad + \bigg(\frac{16\gamma^2 \eta^2 (L+\lambda)^2}{\mu^2} + \frac{(8+16\gamma^2)\lambda^2 \eta^2}{\mu^2}\bigg) \sigma_d^2 \bigg]\notag\\
    & \quad + \bigg(\frac{2\lambda \eta^2}{n} + \frac{4\gamma^2 L^2 \eta^2}{\lambda n}\bigg) \sum_{i \in [n]} \EE[\|\nabla \tilde{\cL}_i (\wb^t) \|^2] + \frac{4\gamma^2}{\lambda} \frac{1}{N} \sum_{i \in [n]} \EE[\|\nabla \cL_i(\wb^t)\|^2],
\end{align}
where the first and last inequalities hold by Cauchy inequality, and the second one holds by $L$-smoothness and Cauchy inequality as well. The third one is due to the bound for $\|\wb_i^{t,K} - \wb^t\|^2$ and the $\gamma$-inexactness of $\eb_i^{t,K}$. Then by the bounded gradient dissimilarity and bounded distilled gradient, organizing pieces, we have 
\begin{align}
    I_3 & \leq \frac{L^2}{\lambda} \bigg[\bigg(\frac{4+8\gamma^2}{\mu^2} + \frac{16\gamma^2 \eta^2 (L+\lambda)^2}{\mu^2} + \frac{(8+16\gamma^2)\lambda^2 \eta^2}{\mu^2}\bigg) \EE[\|\nabla \cL(\wb^t)\|^2] \notag\\
    & \quad + \bigg(\frac{16\gamma^2 \eta^2 (L+\lambda)^2}{\mu^2} + \frac{(8+16\gamma^2)\lambda^2 \eta^2}{\mu^2}\bigg) \sigma_d^2 \bigg]\notag\\
    & \quad + \bigg(\frac{2\lambda \eta^2}{n} + \frac{4\gamma^2 L^2 \eta^2}{\lambda n}\bigg) \sum_{i \in [n]} \EE[\|\nabla \tilde{\cL}_i (\wb^t) \|^2] + \frac{4\gamma^2}{\lambda} \frac{1}{N} \sum_{i \in [n]} \EE[\|\nabla \cL_i(\wb^t)\|^2] \notag\\
    & \leq \frac{L^2}{\lambda} \bigg[\bigg(\frac{4+8\gamma^2}{\mu^2} + \frac{16\gamma^2 \eta^2 (L+\lambda)^2}{\mu^2} + \frac{(8+16\gamma^2)\lambda^2 \eta^2}{\mu^2}\bigg) \EE[\|\nabla \cL(\wb^t)\|^2] \notag\\
    & \quad + \bigg(\frac{16\gamma^2 \eta^2 (L+\lambda)^2}{\mu^2} + \frac{(8+16\gamma^2)\lambda^2 \eta^2}{\mu^2}\bigg) \sigma_d^2 \bigg] + \bigg(4\lambda \eta^2 + \frac{8\gamma^2 L^2 \eta^2}{\lambda}\bigg) \sigma_d^2\notag\\
    & \quad + \bigg(\frac{4\lambda \eta^2}{n} + \frac{8\gamma^2 L^2 \eta^2}{\lambda n} \bigg)\sum_{i \in [n]} \EE[\|\nabla \cL(\wb^t)\|^2] + \frac{4\gamma^2}{\lambda n} \sum_{i \in [n]} \EE[\|\nabla \cL_i(\wb^t)\|^2]\notag\\
    & \leq \bigg[\frac{L^2}{\lambda} \bigg(\frac{4+8\gamma^2}{\mu^2} + \frac{16\gamma^2 \eta^2 (L+\lambda)^2}{\mu^2} + \frac{(8+16\gamma^2)\lambda^2 \eta^2}{\mu^2}\bigg) +   \bigg(\frac{4\lambda \eta^2}{n} + \frac{8\gamma^2 L^2 \eta^2}{\lambda n} \bigg) + \frac{4\gamma^2}{\lambda n} B^2 \bigg]\EE[\|\nabla \cL(\wb^t)\|^2] \notag\\
    & \quad + \bigg[\frac{L^2}{\lambda} \bigg(\frac{16\gamma^2 \eta^2 (L+\lambda)^2}{\mu^2} + \frac{(8+16\gamma^2)\lambda^2 \eta^2}{\mu^2}\bigg) + \bigg(4\lambda \eta^2 + \frac{8\gamma^2 L^2 \eta^2}{\lambda}\bigg)\bigg] \sigma_d^2.
\end{align}

Therefore, by merging pieces together, we have 
\begin{align}
    & \EE[\cL(\xb_{t+1})] - \EE[\cL(\xb_{t})] \notag\\
    & \leq \bigg[-\frac{1}{2\lambda} + \frac{L^2}{\lambda} \bigg(\frac{4+8\gamma^2}{\mu^2} + \frac{16\gamma^2 \eta^2 (L+\lambda)^2}{\mu^2} + \frac{(8+16\gamma^2)\lambda^2 \eta^2}{\mu^2}\bigg) + \frac{4\lambda \eta^2}{n} + \frac{8\gamma^2 L^2 \eta^2}{\lambda n} + \frac{4\gamma^2 B^2 }{\lambda n} \notag\\
    & \quad + \frac{L (2+4\gamma^2)}{\mu^2} + \frac{8L \gamma^2 \eta^2 (L+\lambda)^2}{\mu^2} + \frac{L (4+8\gamma^2)\lambda^2 \eta^2}{\mu^2}\bigg] \EE\|\nabla \cL(\wb^t)\|^2 \notag\\
    & \quad + \bigg[\frac{L^2}{\lambda} \bigg(\frac{16\gamma^2 \eta^2 (L+\lambda)^2}{\mu^2} + \frac{(8+16\gamma^2)\lambda^2 \eta^2}{\mu^2}\bigg) + 4\lambda \eta^2 + \frac{8\gamma^2 L^2 \eta^2}{\lambda} \notag\\
    & \quad + L\bigg(\frac{8\gamma^2 \eta^2 (L+\lambda)^2}{\mu^2} + \frac{(4+8\gamma^2)\lambda^2 \eta^2}{\mu^2}\bigg) \bigg] \sigma_d^2,
\end{align}
summing up from $t=0$ to $T-1$, we have 
\begin{align}
    & \EE[\cL(\xb_{T})] - \EE[\cL(\xb_{0})]\notag\\
    & \leq \bigg[-\frac{1}{2\lambda} + \frac{L^2}{\lambda} \bigg(\frac{4+8\gamma^2}{\mu^2} + \frac{16\gamma^2 \eta^2 (L+\lambda)^2}{\mu^2} + \frac{(8+16\gamma^2)\lambda^2 \eta^2}{\mu^2}\bigg) + \frac{4\lambda \eta^2}{n} + \frac{8\gamma^2 L^2 \eta^2}{\lambda n} + \frac{4\gamma^2 B^2 }{\lambda n} \notag\\
    & \quad + \frac{L (2+4\gamma^2)}{\mu^2} + \frac{8L \gamma^2 \eta^2 (L+\lambda)^2}{\mu^2} + \frac{L (4+8\gamma^2)\lambda^2 \eta^2}{\mu^2}\bigg] \sum_{t=0}^{T-1}\EE[\|\nabla \cL(\wb^t)\|^2] \notag\\
    & \quad + T\bigg[\frac{L^2}{\lambda} \bigg(\frac{16\gamma^2 \eta^2 (L+\lambda)^2}{\mu^2} + \frac{(8+16\gamma^2)\lambda^2 \eta^2}{\mu^2}\bigg) + 4\lambda \eta^2 + \frac{8\gamma^2 L^2 \eta^2}{\lambda} \\\notag
    & \quad + L\bigg(\frac{8\gamma^2 \eta^2 (L+\lambda)^2}{\mu^2} + \frac{(4+8\gamma^2)\lambda^2 \eta^2}{\mu^2}\bigg) \bigg] \sigma_d^2,
\end{align}
then 
\begin{align}
    & \bigg[\frac{1}{2\lambda} - \frac{L^2}{\lambda} \bigg(\frac{4+8\gamma^2}{\mu^2} + \frac{16\gamma^2 \eta^2 (L+\lambda)^2}{\mu^2} + \frac{(8+16\gamma^2)\lambda^2 \eta^2}{\mu^2}\bigg) - \frac{4\lambda \eta^2}{n} - \frac{8\gamma^2 L^2 \eta^2}{\lambda n} - \frac{4\gamma^2 B^2 }{\lambda n} \notag\\
    & \quad - \frac{L (2+4\gamma^2)}{\mu^2} - \frac{8L \gamma^2 \eta^2 (L+\lambda)^2}{\mu^2} - \frac{L (4+8\gamma^2)\lambda^2 \eta^2}{\mu^2}\bigg] \sum_{t=0}^{T-1}\EE[\|\nabla \cL(\wb^t)\|^2] \notag\\
    & \leq \EE[\cL(\xb_{0})] - \EE[\cL(\xb_{T})] + T\bigg[\frac{L^2}{\lambda} \bigg(\frac{16\gamma^2 \eta^2 (L+\lambda)^2}{\mu^2} + \frac{(8+16\gamma^2)\lambda^2 \eta^2}{\mu^2}\bigg) \notag\\
    & \quad + 4\lambda \eta^2 + \frac{8\gamma^2 L^2 \eta^2}{\lambda} + L\bigg(\frac{8\gamma^2 \eta^2 (L+\lambda)^2}{\mu^2} + \frac{(4+8\gamma^2)\lambda^2 \eta^2}{\mu^2}\bigg) \bigg] \sigma_d^2,
\end{align}
hence we have 
\begin{align}
    & \frac{1}{T}\sum_{t=0}^{T-1}\EE[\|\nabla \cL(\wb^t)\|^2] \notag\\
    & \leq \bigg[\frac{1}{2\lambda} - \frac{L^2}{\lambda} \bigg(\frac{4+8\gamma^2}{\mu^2} + \frac{16\gamma^2 \eta^2 (L+\lambda)^2}{\mu^2} + \frac{(8+16\gamma^2)\lambda^2 \eta^2}{\mu^2}\bigg) - \frac{4\lambda \eta^2}{n} - \frac{8\gamma^2 L^2 \eta^2}{\lambda n} - \frac{4\gamma^2 B^2 }{\lambda n} \notag\\
    & \quad - \frac{L (2+4\gamma^2)}{\mu^2} - \frac{8L \gamma^2 \eta^2 (L+\lambda)^2}{\mu^2} - \frac{L (4+8\gamma^2)\lambda^2 \eta^2}{\mu^2}\bigg]^{-1} \frac{\EE[\cL(\xb_{0})] - \EE[\cL(\xb_{T})]}{T} \notag\\
    & \quad + \bigg[\frac{L^2}{\lambda} \bigg(\frac{16\gamma^2 \eta^2 (L+\lambda)^2}{\mu^2} + \frac{(8+16\gamma^2)\lambda^2 \eta^2}{\mu^2}\bigg) + 4\lambda \eta^2 + \frac{8\gamma^2 L^2 \eta^2}{\lambda} \notag\\
    & \quad + L\bigg(\frac{8\gamma^2 \eta^2 (L+\lambda)^2}{\mu^2} + \frac{(4+8\gamma^2)\lambda^2 \eta^2}{\mu^2}\bigg) \bigg] \sigma_d^2.
\end{align}
If we assume that $\eta \leq \frac{1}{2\lambda }$ and $\gamma \to 0$, then we have 
\begin{align}
    & \bigg[\frac{1}{2\lambda} - \frac{L^2}{\lambda} \bigg(\frac{4+8\gamma^2}{\mu^2} + \frac{16\gamma^2 \eta^2 (L+\lambda)^2}{\mu^2} + \frac{(8+16\gamma^2)\lambda^2 \eta^2}{\mu^2}\bigg) - \frac{4\lambda \eta^2}{n} - \frac{8\gamma^2 L^2 \eta^2}{\lambda n} - \frac{4\gamma^2 B^2 }{\lambda n} \notag\\
    & \quad - \frac{L (2+4\gamma^2)}{\mu^2} - \frac{8L \gamma^2 \eta^2 (L+\lambda)^2}{\mu^2} - \frac{L (4+8\gamma^2)\lambda^2 \eta^2}{\mu^2}\bigg] \notag\\
    & = \frac{1}{2\lambda} - \frac{4L^2}{\lambda \mu^2} - \frac{8L^2 \lambda \eta^2}{ \mu^2} - \frac{4\lambda \eta^2}{n} - \frac{2L}{\mu^2} - \frac{4L \lambda^2 \eta^2}{\mu^2}- \cO(\gamma^2) \notag\\
    & \geq \frac{1}{2\lambda} - \frac{4L^2}{\lambda \mu^2} - \frac{2L^2 }{ \lambda \mu^2} - \frac{1}{\lambda n} - \frac{2L}{\mu^2} - \frac{4L }{\mu^2} - \cO(\gamma^2)> 0,
\end{align}
and 
\begin{align}
    & \frac{8L^2 \gamma^2}{\lambda \mu^2} + \frac{16\gamma^2 L^2 \eta^2 (L+\lambda)^2}{\lambda \mu^2} + \frac{16\gamma^2 L^2 \eta^2 \lambda^2}{\lambda \mu^2} + \frac{8\gamma^2 L^2 \eta^2}{\lambda n} + \frac{4\gamma^2 B^2 }{\lambda n} + \frac{4 \gamma^2L}{\mu^2} + \frac{8L \gamma^2 \eta^2 (L+\lambda)^2}{\mu^2} + \frac{8L\gamma^2 \lambda^2 \eta^2}{\mu^2} \notag\\
    & \leq \frac{8L^2 \gamma^2}{\lambda \mu^2} + \frac{4\gamma^2 L^2 (L+\lambda)^2}{\lambda^3 \mu^2} + \frac{4\gamma^2 L^2 }{\lambda \mu^2} + \frac{2\gamma^2 L^2 }{\lambda^3 n} + \frac{4\gamma^2 B^2 }{\lambda n} + \frac{4 \gamma^2L}{\mu^2} + \frac{2L \gamma^2 (L+\lambda)^2}{\lambda^2 \mu^2} + \frac{2L\gamma^2}{\mu^2} = \cO(\gamma^2),
\end{align}
hence we need the condition of 
\begin{align}
    &\cO(\gamma^2) \leq \frac{1}{6\lambda}  \text{ and } \frac{1}{6\lambda} - \frac{4L^2}{\lambda \mu^2} - \frac{2L^2 }{ \lambda \mu^2} - \frac{1}{\lambda n} - \frac{2L}{\mu^2} - \frac{4L }{\mu^2} \geq 0 \\
    & \gamma \leq 1/\sqrt{C \lambda},
\end{align}
where $C$ is a numerical constant irrelevant to $L, B, \gamma, \lambda, \eta$, etc.
To obtain the final convergence rate. Reorganize items, we have 
\begin{align}
    \min_{t\in [T]} \EE[\|\nabla \cL(\wb_t)\|_2^2] 
    & \leq \cO\bigg(\frac{\EE[\cL(\wb_{0})] - \EE[\cL(\wb_{T})]}{6\lambda T}\bigg) + \cO(\eta^2 \sigma_d^2).
\end{align}
This concludes the proof.
\end{proof}

\newpage

\section{Additional Experiments}
\subsection{More Experiments on Layer Adaptive $\lambda$}
\label{app:layer}
 In \cref{fig:norm_bar}, we plot the $L_2$ norms of the four different convolutional layers in ConvNet\cite{gidaris2018dynamic}, with the model trained under a 10-clients federated learning setup, and each client's training data sampled from the original CIFAR10 dataset following Dirichlet distribution ($\alpha=0.01$). We can observe that the $L_2$ norm of $\mathbf{w} - \tilde{\mathbf{w}}_i^{t+1}$ between layers varies and it also changes along with different phases of training. For example, at the beginning of training, the fourth layer has the largest $L_2$ norm while the second layer has the largest $L_2$ norm when the training is about to end. 
 Such a phenomenon indicates that it might be helpful to use a layer-adaptive $\lambda$ rather than a constant. We compare our layer adaptive $\lambda$  with the standard fixed $\lambda$. We report results on CIFAR10 and CIFAR100 with the parameter of Dirichlet distribution $\alpha=0.01$ and $0.04$ as shown in \cref{table:lambda}. We can observe that our FedPTR adopting layer adaptive $\lambda$ achieves better performances compared with the fixed $\lambda$ setting. More experiments compared with layer adaptive $\lambda$ and fixed $\lambda$ can be found in the following section.
 
 \begin{figure}[htb]
\begin{center}
\centerline{\includegraphics[width=0.35\columnwidth]{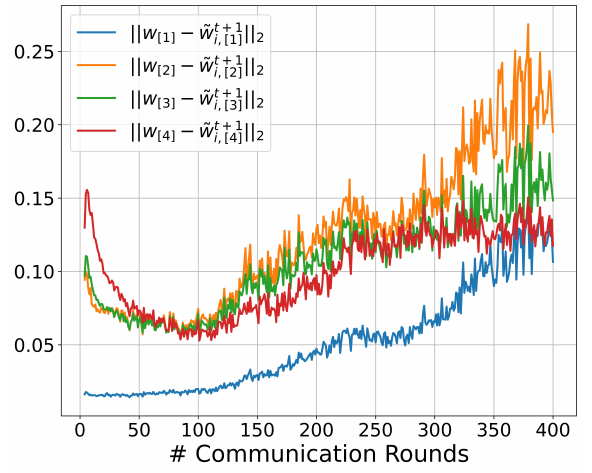}}
\caption{Comparisons of layer-wise $L_2$ norm $\big\lVert \wb_{[j]} -\tilde\wb^{t+1}_{i,[j]} \big\rVert_2$ against the communication rounds, where $j$ denotes $j$-th layer. The layer-wise $L_2$ norm values vary significantly with different layers and communication rounds. }
\label{fig:norm_bar}
\end{center}
\vskip -0.2in
\end{figure}

\begin{table}[htb]
\begin{center}
    \small
    \begin{tabular}{ c | c c  | c c  }
    \toprule
    \multirow{2}{*}{Reg. param.} & \multicolumn{2}{c|}{$\text{Dirichlet } \alpha = 0.01$} & \multicolumn{2}{c}{$\text{Dirichlet } \alpha = 0.04$} \\
    \cmidrule(rl){2-3} \cmidrule(rl){4-5}
    &  CIFAR 10 & CIFAR 100 & CIFAR 10 & CIFAR 100 \\
    \midrule
    Fixed $\lambda$ & 62.10 & 41.26 & 67.68 & 42.41 \\ 
    Layer-adaptive $\lambda$ & 69.04 & 44.98 & 72.78 & 45.87 \\
    \bottomrule
    \end{tabular}
\end{center}
\vskip -0.05in
\caption{Comparisons of layer-adaptive $\lambda$ with fixed $\lambda$ in FedPTR on CIFAR10 and CIFAR100 dataset.}
\label{table:lambda}
\end{table}
\subsection{More Results of Ablation Study}
\label{ap:ab}
\noindent\textbf{Impact of trajectory projecting steps.}
To test sensitivity of the trajectory projection steps $K$, we consider CIIFAR10 with Dirichlet $0.01$. \cref{fig:tilde} shows convergence plots for different $K$ configurations. We can observe that our method is relatively stable when $k=3$ and $k=5$. However, setting $k=1$ would lead to slower convergence and a slight performance decrease, and setting $k=10$ could accelerate the training in the early stage but result in instability and make it difficult to converge.
\begin{figure}[htb]
\begin{center}
\centerline{\includegraphics[width=0.35\columnwidth]{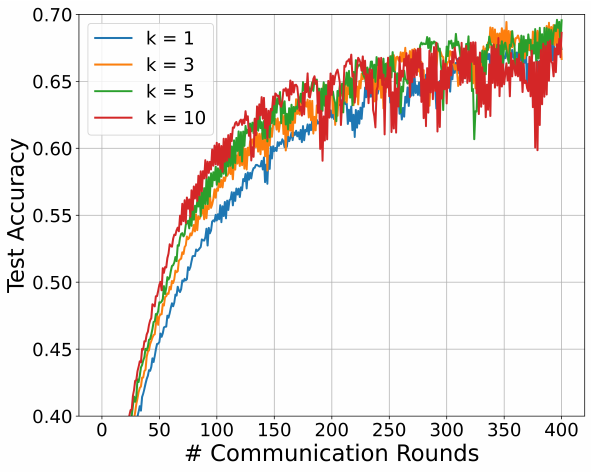}}
\vskip -0.1in
\caption{Comparisons of FedPTR with different trajectory projecting steps $K$ on CIFAR10 with $\text{Dirichlet } \alpha = 0.01$.}
\label{fig:tilde}
\end{center}
\vskip -0.2in
\end{figure}

\noindent\textbf{Different Initialization Strategies. } As presented in \cref{table:random_init},  we compare our strategy that initializes the auxiliary dataset using (partial) local client data with the random initialization strategy.
It can be seen that FedPTR using our proposed initialization strategy obtains better performance on CIFAR10 and CIFAR100 with the improvement of $1.54\%$ and $0.47\%$, respectively. On the other hand, our FedPTR still maintains excellent performance even after adopting the random initialization strategy, which supports the design of server-side FedPTR (FedPTR-S) that performs MTT and optimizes the auxiliary dataset with random initialization in the global server.
\begin{table}[htb]
	\centering 
	\begin{tabular}{l|ccc}
		\toprule
		 {Dataset}  & FMNIST & CIFAR10 & CIFAR100 \\\midrule
             FedPTR - random init & 85.13 &67.50 & 44.51 \\
             FedPTR - our init & 85.11 &69.04 & 44.98 \\
        \bottomrule
	\end{tabular}
 \vskip -0.1in
	\caption{Comparisons of different strategies of auxiliary dataset initialization on CIFAR10 with Dirichlet $\alpha = 0.01$.
	}
	\label{table:random_init}
        \vskip -0.1in
\end{table}

\noindent\textbf{Impact of trajectory matching steps.}
In \cref{table:mtt_iteration_cifar10}, 
we present the performance of our server's side method FedPTR-S on CIFAR10 and CIFAR100, respectively, with varying total trajectory matching steps $H$ used in \cref{MTT}. Specifically, we consider 10 clients, full participation setting with two different non-iid data participation Dirichlet $\alpha=0.01$ and Dirichlet $\alpha=0.04$. We can see that our method is stable with respect to different trajectory matching steps. Even if we reduce the steps from 20 to 5, the test accuracy obtained on CIFAR10 and CIFAR100 only decreases by $1.16\%$ and $0.5\%$, respectively.
\begin{table}[htb]
\begin{center}
    \begin{tabular}{c| c | c c c c }
    \toprule
    \multirow{5}{*}{CIFAR10}&\multirow{2}{*}{Dirichlet $\alpha$} & \multicolumn{4}{c}{Trajectory matching steps $H$ } \\
    \cmidrule{3-6}
     & & 5 & 10 & 20 & 30 \\
    \cmidrule{2-6}
     & $\alpha=0.01$  & $67.00$ & $68.12$ & $68.16$ &$67.04$\\
     & $\alpha=0.04$  & $72.85$ & $71.44$ & $73.07$ &$73.20$\\
    \midrule
     \multirow{2}{*}{CIFAR100} & $\alpha=0.01$  & $44.56$ & $44.53$ & $44.82$ &$44.89$\\
     & $\alpha=0.04$  & $45.60$ & $45.89$ & $46.10$ &$45.85$\\
     \bottomrule
    \end{tabular}
\end{center}
\vskip -0.1in
\caption{Comparisons of FedPTR-S with different trajectory matching steps on CIFAR10 and CIFAR100 with $\text{Dirichlet } \alpha = 0.01$ and $0.04$.}
\label{table:mtt_iteration_cifar10}
\vskip -0.1in
\end{table}

\noindent\textbf{Layer-adaptive $\lambda$ V.S. Fixed $\lambda$.} In \cref{table:lambda}, we have already shown the effectiveness of layer-adaptive $\lambda$, and here we give a further analysis. As shown in \cref{fig:norm}, we present the complete convergence plots of FedPTR with layer-adaptive $\lambda$ and fixed $\lambda$ on CIFAR10 and CIFAR100 with Dirichlet $\alpha=0.01$ and Dirichlet $\alpha=0.04$. It can be seen that adopting layer-adaptive $\lambda$ slightly falls behind at the beginning of the training but obtains better performance in the late phase of the training, indicating that layer-adaptive $\lambda$ can help converge to a better stationary point.
\begin{figure}[htb]
\centering
    \subfigure[CIFAR10, Dir$(0.01)$]{
        \includegraphics[width=3.2cm]{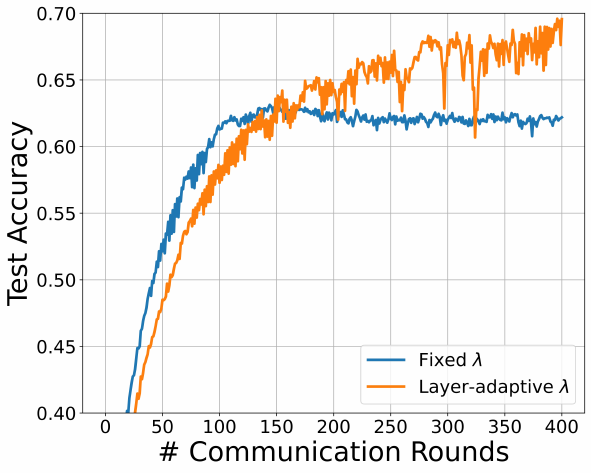}
    }
    \subfigure[CIFAR10, Dir$(0.04)$]{
	\includegraphics[width=3.2cm]{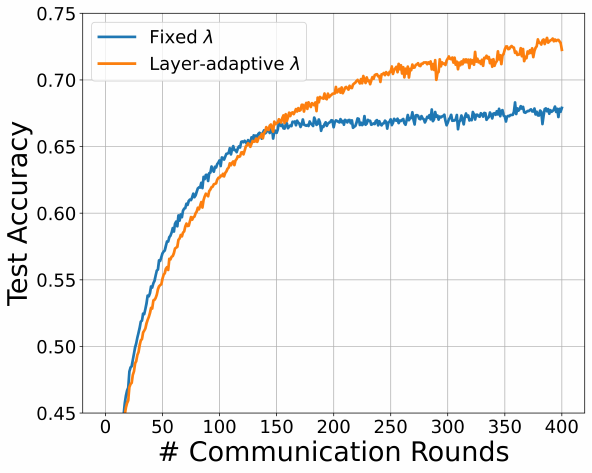}
    }
    \subfigure[CIFAR100, Dir$(0.01)$]{
    	\includegraphics[width=3.2cm]{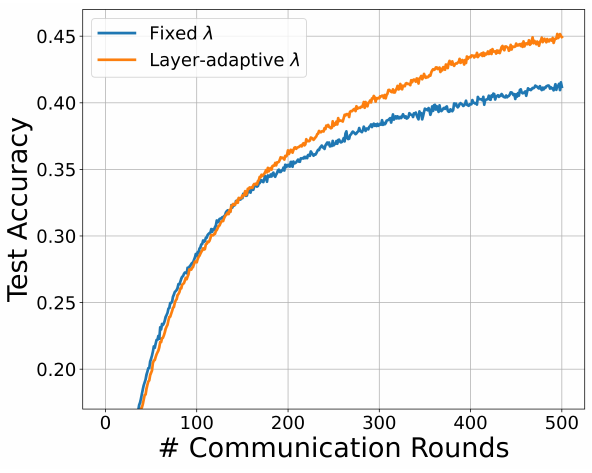}
    }
    \subfigure[CIFAR100, Dir$(0.04)$]{
	\includegraphics[width=3.2cm]{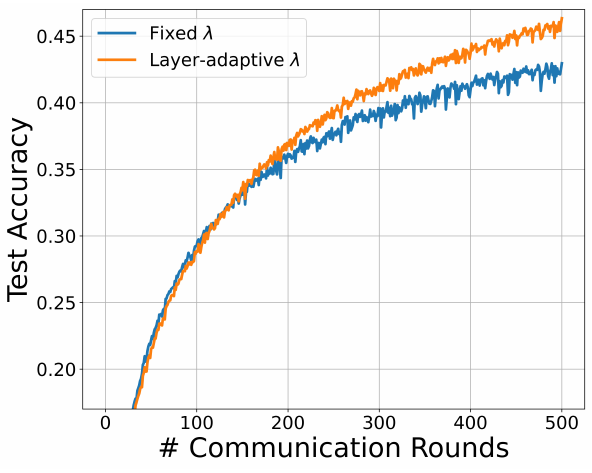}
    }
    \vskip -0.1in
    \caption{Test accuracy against the communication rounds, obtained by FedPTR with different regularizations.}
    \label{fig:norm}
\vskip 0.05in
\end{figure}

\subsection{More Results of Reducing the Updates of Auxiliary Dataset} \label{ap:more_updates}
For the update frequency, we have conducted ablation studies on client's side as shown in Figure \ref{fig:reduce_updates}. Following the same setting, here we present our experiments on server's side in Table \ref{table:more_freq}, which shows that the server-side method FedPTR-S can still maintain a relatively stable performance with reduced updates of auxiliary datasets.
\begin{table}[H]
\begin{center}
    \begin{tabular}{ ccc }
    \toprule
     Update Frequency & Test Acc \\
     \midrule
     per $1$ round & $68.16$ \\
     per $4$ round &  $65.42$  \\ 
     per $10$ round & $65.97$ \\
     per $40$ round & $65.82$ \\
     one-shot &  $65.34$\\
    \bottomrule
    \end{tabular}
\vskip -0.1in
\caption{Comparions of FedPTR-S with different update frequencies of the auxiliary dataset on CIFAR10 with $\text{Dirichlet } \alpha = 0.01$.}
\label{table:more_freq}
\end{center}
\vskip -0.2in
\end{table}

\subsection{More Results of Changing the Size of Auxiliary Dataset.} \label{ap:more_size_data}
We have conducted comparisons of the client's side method FedPTR with different sizes of the auxiliary dataset on CIFAR10 with Dirichlet $\alpha=0.01$ as shown in Table \ref{table:size_cifar10}. Here we present the results on server's side method FedPTR-S with the same setting as shown in Table \ref{table:more_size}. We can observe FedPTR-S could obtain better performance with more auxiliary data, similar to FedPTR.
\begin{table}[htb]
\begin{center}
    \begin{tabular}{ ccc }
    \toprule
     Size of $\mathcal{\tilde{D}}$ & Test Acc \\
     \midrule
     10 & $67.47$ \\
     50 &  $67.64$  \\ 
     100 & $68.16$ \\
     150 & $69.11$ \\
     200 &  $69.75$\\
    \bottomrule
    \end{tabular}
    \vskip -0.1in
\caption{Comparions of FedPTR-S with different sizes of the auxiliary dataset on CIFAR10 and CIFAR100 with $\text{Dir} \alpha = 0.01$.}
\label{table:more_size}
\end{center}
\vskip -0.2in
\end{table}



\subsection{Experiments on TinyImagenet} \label{ap:tiny}
We additionally conduct experiments to test our method on larger TinyImageNet data to demonstrate that our method can indeed extend to larger-scale experiments. We summarize the experimental results with 10 clients and Dir $\alpha =0.01$ in Table \ref{table:more_tinyimgnet}. The results suggest that our method still outperforms other baselines.
\begin{table}[htb]
\begin{center}
    \begin{tabular}{ c c c c c c c  }
    \toprule
     FedAvg & FedProx & Scaffold & FedDyn & FedDC & FedPTR & FedPTR-S \\
     \midrule
     $27.68$ & $27.11$ & $21.55$ & $27.49$ & $23.81$ & $\underline{28.6}$ & $\mathbf{29.6}$ \\
    \bottomrule
    \end{tabular}
\end{center}
\vskip -0.1in
\caption{Comparison of test accuracy for FedPTR (FedPTR-S) and other baselines on TinyImageNet.}
\label{table:more_tinyimgnet}
\end{table}

\section{Privacy leakage of FedPTR}
For privacy leakage of our proposed method, we argue that although we did not add any privacy protection techniques into our framework, FedPTR does not incur extra privacy leakage compared to FedAvg\cite{mcmahan2017communication}/FedDyn\cite{acar2021federated} since all the communications between clients and server are model weights/updates, just like FedAvg\cite{mcmahan2017communication}/FedDyn\cite{acar2021federated}. To empirically demonstrate that FedPTR does not incur extra privacy leakages, we further conduct membership inference attack (MIA) to the final model trained by different methods. We show the MIA results of 10 clients FL with $\alpha = 0.04$ on CIFAR100 dataset in the \cref{tab:mia} using two well-recognized MIA methods: 1) loss-based method and 2) Likelihood ratio attack (LiRA). Here accuracy means whether the MIA method successfully predicts whether a sample is in the original training set or not. In other words, a MIA test accuracy of a method that is closer to 50\% (random guess) indicates this method barely leaks private information. We can observe that FedPTR achieves a similar level of privacy leakage as other baselines such as FedAvg\cite{mcmahan2017communication}/FedDyn\cite{acar2021federated}, as expected.
\begin{table}[htb]
    \centering
    \begin{tabular}{c c c c}
    \toprule
    & FedAvg & FedDyn & FedPTR \\
    \midrule
    loss-based & 56.37 & 61.67 & 57.46 \\
    LIRA & 61.06 & 68.29& 65.59 \\
    \bottomrule
    \end{tabular}
\vskip -0.1in
    \caption{Membership inference attack for global model with $\alpha = 0.04$ on CIFAR100 dataset.}
    \label{tab:mia}
\end{table}

\section{Detailed Description of Data Sampling}
As we have mentioned in experimental settings, for all experiments, we adopt Dirichlet data sampling methods, which are commonly used in federated learning papers\cite{acar2021federated,wang2020federated}. For Dirichlet sampling methods, each client $j$ is allocated a proportion of the samples from each label $k$ according to Dirichlet distribution. Specifically, we sample $p_k \sim Dir_N(\alpha)$ and allocate a $p_{kj}$ proportion of data samples of class $k$ to client $j$, where $Dir_N(\alpha)$ denotes the Dirichlet distribution and $\alpha > 0$ is a concentration parameter. A smaller $\alpha$ indicates a higher non-i.i.d. degree when sampling data, and we focused on the high-level data heterogeneity with $\alpha = 0.01$ and $\alpha = 0.04$. \cref{tab:data_par} shows a typical Dirichlet sampling result for different settings. We can observe that when choosing 40 clients on CIFAR10, it can easily assign 0 data samples to certain clients due to low value and a small number of classes. That is the reason why we do not conclude experiments with more clients on CIFAR10 and FashionMNIST.

\begin{table}[H]
    \centering
    \begin{tabular}{c c c c}
    \toprule
    Dataset & $\alpha$ & \# clients & \# samples in each client \\
    \midrule
    FMNIST & 0.01 & 10 & $12032, 4244, 6003, 6037, 6002, 6165, 5963, 5760, 2602, 5192$ \\
    \midrule
    CIFAR10 & 0.01 & 10& $2972, 9419, 2936, 5779, 4111, 5003, 4949, 4979, 5615, 4237$ \\
    \midrule
    CIFAR100 & 0.01 & 10 & $2690, 5209, 4764, 6463, 4467, 5679, 3983, 5910, 5497, 5338$ \\
    \midrule
    CIFAR100 & 0.01 & 40 & \makecell[c]{ $735, 1371, 1365,  761, 1123, 2194, 2399,  531,  848, 1549,$ \\ $2378,  569, 2007,  984, 518,  608, 1250, 780, 2601, 2784,$ \\ $1408, 1831, 1108, 2166,  523,  814, 1281,  579, 1067, 1172,$ \\ $803, 1085,  902, 2603, 1237,  940,  588,  683, 1246,  609$} \\
    \midrule
    CIFAR10 & $0.01$ & $40$ & \makecell[c]{$0, 3133, 0,  274,  373, 0,  296, 1709, 1114, 4636,$ \\ $5287, 0, 1469, 1684, 0, 415, 3, 3950, 718, 46, $ \\  $ 3355, 29, 7318, 1033, 0, 0, 918, 0, 1342, 0,$ \\ $1, 0, 4979, 101, 1397, 1615, 5, 2798, 2, 0$} \\
    \bottomrule
    \end{tabular}
\vskip -0.1in
    \caption{Dirichlet sampling result for different settings.}
    \label{tab:data_par}
\vskip -0.1in
\end{table}

\section{Naive Attempt: Combining Distilled Dataset with Local Data}
In this section, we present a naive attempt of directly combining distilled dataset with local data and analyze its failed case. Note that MTT\cite{cazenavette2022dataset} is originally proposed for distilling a small synthetic dataset. Thus a trivial idea for tackling data heterogeneity is to apply MTT\cite{cazenavette2022dataset} to obtain the synthetic dataset and combine those distilled data with the available local data to train FedProx\cite{li2020federatedprox} for relieving the data heterogeneity. However, this trivial design is not ideal for addressing the data heterogeneity issue especially when data are highly non-i.i.d. distributed among clients, as demonstrated in \cref{tab:direct}. The result shows that simply adding the synthetic data by MTT\cite{cazenavette2022dataset} for FedProx training obtains even worse performance than FedAvg. This is primarily due to the fact that the synthetic data obtained by MTT\cite{cazenavette2022dataset} under federated settings (original MTT\cite{cazenavette2022dataset} can use multiple training trajectories and apply restart strategy which is not realistic under federated setting) may not be ideal enough to fully represent the information inside the original training data. Therefore, directly training on the synthetic data may actually make the training worse. While in our case, we only require the auxiliary data produced by MTT\cite{cazenavette2022dataset} to guide the next-step training trajectory prediction (which is exactly what our MTT\cite{cazenavette2022dataset} is trained for, i.e., to match recent training trajectories). Thus we believe that our methods are fundamentally different from simply combining the previous methods FedProx\cite{li2020federatedprox} and MTT\cite{cazenavette2022dataset}.
\begin{table}[H]
    \centering
    \begin{tabular}{c c c c}
    \toprule
    & FMNIST & CIFAR10 & CIFAR100 \\
    \midrule
    Extra synthetic data & 76.23 & 50.08 & 27.78 \\
    FedAvg & 83.69 & 55.45& 37.32 \\
    FedPTR & 85.11 & 69.04 & 44.98 \\
    FedPTR-S & 85.48 & 68.16 & 44.82 \\
    \bottomrule
    \end{tabular}
\vskip -0.1in
    \caption{Comparsion of directly using auxiliary dataset and FedPTR on three datasets.}
    \label{tab:direct}
\vskip -0.1in
\end{table}

\end{document}